\documentclass{article}

\usepackage{blindtext}

\usepackage{arxiv}
\usepackage{amsthm}
\newtheorem{thm}{Theorem}

\newtheorem{lem}[thm]{Lemma}
\newtheorem{prop}[thm]{Proposition}

\usepackage{enumitem}

\usepackage[linesnumbered,ruled]{algorithm2e}%
\usepackage{algorithmic}

\usepackage{float}
\DontPrintSemicolon
\usepackage{xcolor}

\usepackage{color}
\usepackage{amssymb}
\usepackage{dsfont}
\usepackage{graphicx}
\usepackage[colorinlistoftodos]{todonotes}
\usepackage[colorlinks=true, allcolors=blue]{hyperref}

 \usepackage{booktabs}
\usepackage{booktabs,makecell}

\usepackage{tikz}
\usetikzlibrary{positioning}
\usetikzlibrary{arrows.meta}
\usepackage{graphicx,wrapfig,lipsum}

\usepackage[utf8]{inputenc}
\usepackage[utf8]{inputenc} 
\usepackage[T1]{fontenc}    
\usepackage{hyperref}       
\usepackage{url}            
\usepackage{booktabs}       
\usepackage{amsfonts}       
\usepackage{nicefrac}       
\usepackage{microtype}      
\usepackage{lipsum}		
\usepackage{graphicx}
\usepackage{natbib}
\usepackage{doi}
\usepackage{amsmath}

\title{On  mixing Times of Metropolized Algorithm with Optimization Step (MAO): A New Framework}

\author{ EL Mahdi Khribch \\
	Department of Statistics\\
	Oxford University\\
	Oxford, OX1 3LB \\
	\texttt{el.khribch@kellogg.ox.ac.uk} \\
	\And
    George Deligiannidis \\
	Department of Statistics\\
	Oxford University\\
	Oxford, OX1 3LB \\
	\texttt{george.deligiannidis@stats.ox.ac.uk} \\
	\And
   Daniel Paulin \\
 School of Mathematics \\
	University of Edinburgh.\\
	 Edinburgh EH9 3FD \\
	\texttt{dpaulin@ed.ac.uk} \\
}

\renewcommand{\shorttitle}{\textit{arXiv} Template}
\hypersetup{
pdftitle={A template for the arxiv style},
pdfsubject={q-bio.NC, q-bio.QM},
pdfauthor={ELMAhdi khribch, Elias D.~Striatum},
pdfkeywords={MCMC  algorithms, Metropolis-Hastings algorithms},
}

\begin{document}
\setcounter{page}{1}
\maketitle
\begin{abstract}
In this paper, we take a closer look at the problem of sampling from a class of distributions with thin tails supported on $\mathbb{R}^d$ and make two primary contributions. First, we propose a new algorithm capable of working in regimes where the Metropolis-adjusted Langevin algorithm (MALA) is not converging or lacking in theoretical guarantees~\cite{tweediemala}. In addition, we derive upper bounds on the mixing time of the Markov chain generated by the said algorithm. The details of the proof follow the same technical methods used in previous works~\cite{dwivedi2018log}, and~\cite{chewi2020}. We also leverage the framework developed by~\cite{chenhmc} to extend the dependency on the warmness leveraging log-isoperimetric inequality results.
\end{abstract}


\section{Introduction}
The ability to draw samples from a distribution is at the heart of many applications within the Bayesian paradigm and, more generally, in computational statistics. Markov Chain Monte Carlo pioneered by ~\cite{metropolis1953equation}, is often considered among practitioners as the default method for obtaining samples from distributions in a high-dimensional setting. 
In practice, variants of the Metropolis-Hastings enjoy tremendous success, notably in posterior exploration within a Bayesian setting  ~\cite{carpenter2017stan,smith2014mamba}. In addition, Monte Carlo methods are commonly deployed in several applications: estimating the posterior mean, computing expectations of quantities of interest, and volumes of particular sets.
Recently the research community has been interested in a noticeable manner in sampling methods and their interplay with the more established field of optimization~\cite{Ma}. More specifically, due to the asymptotic nature of MCMC methods, a more tractable characterization of the dimension dependency of the convergence is an essential step in order to develop a better understanding of the convergence of this class of algorithms and to practical guidelines for practitioners.
\subsection{Problem set up and Related work }
The standard formulation of non asymptotic bounds on sampling methods considers the task of drawing samples from a target distribution $\Pi$  with with density  $\pi$ supported on $\mathbb{R}^d$  such that 
 \begin{equation}
 \label{EqnDefnTarget}
     \pi(x)=\frac{e^{-f(x)}}{\displaystyle \int_{\mathbb{R}^d} e^{-f(y)} dy }  \forall x\in \mathbb {R}^d
 \end{equation}
~\cite{Roberts97}  was the first to study the dimension dependence of Random Walk Metropolis algorithm (RWM) for a class of potentials $f:\mathbb{R}^d\rightarrow \mathbb{R}$ in the case of a chain of i.i.d samples from a strongly convex and smooth potential, this resulted in the establishment of an asymptotic framework leading to the scaling limit for RWM when the dimension tends to infinity with a  step size $h\approx d^{-1}$, where d is the dimension of the problem. Subsequently \cite{roberts2001optimal} extended the same approach  to MALA, concluding that the dimension dependence when considering  MALA is of $d^{-\frac{1}{3}}$, for a step size of order $h\approx d^{-\frac{1}{3}}$. Recent years have witnessed a surge of interest as to the non-asymptotic performance of sampling algorithms- (mainly MALA) - over the class of smooth and strongly convex potentials. We note the work of~\cite{dwivedi2018log} and ~\cite{chenhmc}  that shows that RMW can achieve a mixing time of order $\mathcal{O}(d\log(\frac{1}{\epsilon}))$  within $\epsilon$ error in $\chi^2$-distance, thus providing an explicit  non-asymptotic bound on the scaling limit of \cite{roberts2001optimal}, while ~\cite{chewi2020}  proved that the optimal mixing time of MALA over said class of potentials with a warm start is of order $\tilde{\mathcal{O}}(\sqrt{d})$. However, there are instances in which MALA does not satisfy geometric ergodicity as established by ~\cite{tweediemala}, and in such regimes, we expect  MALA to fail, more particularly in this work, we take a closer look at the class of thin tailed distributions, that is distributions with tails that decay with a faster rate than a Gaussian.
\subsection{ On mixing times and initialization}
The recent work by ~\cite{chenhmc} adapts several key techniques for establishing the convergence of continuous-state Markov chains inspired by the large body of literature on discrete-state Markov(~\cite{lovasz1993randomsurvey,aldous2002reversible}). Their results show an improvement in the dependency of the mixing time on the initial distribution. More specifically, the framework developed proves   that the logarithmic dependency of the mixing time of a
Markov chain on the warmness parameter (see section~\ref{par:initial_distribution_} for a formal definition) of the initial distribution can be enhanced if certain assumptions hold to doubly logarithmic. The novelty in driving this improvement is the use of log-Sobolev inequalities instead of the usual isoperimetric inequality, mainly log-isoperimetric inequality; this result has improved mixing times in their treatment of Metropolized Hamiltonian Monte Carlo. In this work, we will make use of the improved dependence on warmness under the framework detailed in the work of ~\cite{chenhmc} in the case of thin-tailed potential targets.

\subsection{Contributions of our current work}
In this paper, we propose a new MCMC method called Metropolis algorithm with optimization step (MAO), and prove two main results. First, we derive a non-asymptotic upper bound on the mixing time of our new Algorithm for the class of thin tailed densities, and we improve the dependency on the warmness using the framework developed by ~\cite{chenhmc}. In particular, we recover the rate derived by ~\cite{dwivedi2018log} for MALA.
Our second contribution consists of deriving improved mixing times for potentially thin tailed distributions that verify tighter concentration on a ball of radius $\propto d^{1/ \alpha}$ where $\alpha \geq 2$, with $\alpha=2$ corresponding to Gaussian concentration. These contributions show that unlike MALA, the MAO algorithm is well suited for exploring distributions with thin tails (as the acceptance rate stays high everywhere in the state space). Finally, we also include extensive numerical simulations that are consistent with our theoretical results and show orders of magnitude improvements in sampling efficiency for thin tailed potentials.
\section{Background}
\label{sec:problem_set_up}
In this section, we detail some background on Markov chain Monte Carlo methods. Then we describe the set of assumptions that will be considered on the target distribution along with a set of new terminology. In addition, we describe some of the most notable algorithms belonging to the class of Metropolis-Hastings. Most notably, we describe the Metropolized random walk (MRW) and the Metropolis-adjusted Langevin algorithm (MALA). We also restate the rates of convergence of existing random walks when targeting log-concave distributions.

\subsection{Background on Markov chain Monte Carlo}
\label{sub:monte_carlo_markov_chain_methods}
When considering a target distribution $\Pi$  with a density $\pi$, we are interested in the task of drawing samples from $\Pi$. In practice, the samples obtained can be used to approximate expectations of random quantities.Indeed, for a given function  \mbox{$h: \mathcal{X} \rightarrow \mathbb{R}$} we seek to 
approximate $\mathbb{E}_{\pi}[{h(X)}] =
\int_{\mathcal{X}} h(x) \pi(x) dx$.  This particular task is untractable analytically due to the well-known curse of dimensionality.
The intuition behind Monte Carlo simulation methods is to generate  i.i.d. random samples $Z_i \sim
\Pi$ for $i = 1, \ldots, N$, the
random variable $ \hat{Z}_{n} := \frac{1}{N} \sum_{i=1}^N h(Z_i)$ is an unbiased estimator of the expectation. In practice it is often a challenging task to draw i.i.d. samples $Z_i$, in particular this problem is notoriously difficult in a high dimensional setting.

Markov chain Monte Carlo (MCMC)  is rooted in the principle of constructing an irreducible, aperiodic discrete-time Markov chain with an initial distribution $\mu_0$ and whose stationary distribution is the desired target distribution $\Pi$.

When studying MCMC methods, we are interested in studying the design of such chains and in the properties of their convergence, mainly the number of steps it takes the chain to converge to the stationary distribution. Such questions have been investigated throughout the years, culminating in a considerable body of research. We refer to the reviews by  ~\cite{tierney1994markov,smith1993bayesian,roberts2004general} for a comprehensive review. In this work, we are interested
in comparing the performance of the acclaimed Metropolis-Hastings adjusted Markov chain algorithm (MALA) to our proposed Algorithm (MAO).  Our main goal is to tackle the second question in the case of our proposed Algorithm, particularly via establishing an explicit non-asymptotic mixing-time bound on a specific class of target distributions and thereby characterizing how it converges in regimes where MALA can potentially fail.
\subsection{Terminology and Assumptions.}
\label{sub:assumptions}
From now on, we assume the familiarity of the reader with elementary notions on Markov chains.

We study the task of sampling from a distribution $\Pi$ supported on $\mathbb{R}^d$  with density $\pi$, where the density $\pi$ is of the form $\pi(x)\propto e^{-f(x)}$,  $f$ is referred to as  potential throughout this work. We will make different assumptions on the potential $f$. Finally ,we make the assumption that  $f(0)=\textrm{arg}\min f(x) = 0$ so that $\nabla f(0)=0$.

Throughout this work, we will make different types of assumptions on the potential $f$. However, we first introduce notations and certain regularity assumptions 
\paragraph{Regularity Assumptions.}
\label{par:regularity_conditions_}
\begin{subequations}
A potential $f$ is said to be :
  \begin{align}
    \label{eq:assumption_smoothness}
    L\text{-smooth}: \quad\ \  f(y)-f(x)-\nabla f(x)^\top(y-x)\leq\frac{L}{2}\left\|x-y\right\|_{2}^2,\quad \forall  x, y \in
  \mathbb{R}^d
    \\
    \label{eq:assumption_scparam}
    m\text{-strongly convex}: \quad\ \ f(y)-f(x)-\nabla f(x)^\top(y-x)\geq\frac{m}{2}\left\|x-y\right\|_{2}^2,\quad \forall  x, y \in
  \mathbb{R}^d
  \end{align}
\end{subequations}
\\

\paragraph{Thin tailed potentials.}
\label{THin}
For $\alpha\geq 2$, we define a  a class of potentially thin tailed potentials $\mathcal{E}(\alpha)$ where we have for a given function $f\in\mathcal{E}(\alpha)$ we have $f:
\mathbb{R}^d\rightarrow \mathbb{R}$ continuously differentiable, and
\begin{subequations}
\begin{align}
    \label{eq:alpha_f}
    \textrm{(i)}
    \quad\ \ 
    \left\|\nabla f(x)\right\|\leq K_1(1+\left\|x-x^*\right\|^{\alpha-1}),\\
     \label{eq:nabla_f}
     \textrm{(ii)} \quad\ \ \left\|\nabla^2 f(x)\right\|\leq K_2(1+\left\|x\right\|^{\alpha-2}),
\end{align}
\end{subequations}
where  $K_1,K_2>0$ and the minimizer $x^*$ depend on $f$. We also define the subclass $\mathcal{E}(\alpha,m)\subset \mathcal{E}(\alpha)$ of $m$-Strongly convex potentials within the $\mathcal{E}(\alpha)$ set.

Note that assumption (ii) is redundant but we state it in the following form for practical considerations.
\paragraph{Assumptions on $\Pi$.}
\label{par:assumptions_on_the_target_distribution_}
We provide here after the two sets of sets of assumptions that we impose on the target distribution:
\begin{enumerate}[label=(\Alph*)]
  \item \label{itm:assumptionA} 
  A target distribution $\Pi$ is said to be $(\alpha,m)$-\textit{strongly log-concave} if its potential $f$ is
     $m$-strongly convex ~\eqref{eq:assumption_scparam} and 
     $f \in\mathcal{E}(\alpha)$. We call the set of such potential functions $\mathcal{E}(\alpha,m)$.
  \item \label{itm:assumptionB} 
  A target distribution $\Pi$ is said to be $(\alpha,\gamma,m)$-\textit{strongly log-concave} if its potential $f$ satisfies the following
  \begin{enumerate}
     \item $f$ is $m$-strongly convex ~\eqref{eq:assumption_scparam}, 
     \item 
     $f \in\mathcal{E}(\alpha)$,
     \item There is a polylogarithmic function \footnote{polylogarithmic means that $\tau(s)\le p(\log(1/s))$ for some polynomial $p$}
     $\tau:\mathbb{R}_+\to \mathbb{R}_+$
     (depending on $f$) such that for every $s>0$, there is a convex set $\Omega\in \mathbb{R}^d$ containing $x^*$ such that $\Pi(\Omega)\ge 1-s$ and $\textrm{diam}(\Omega)\le \tau(s) d^{\gamma}$. 
     \end{enumerate}
     We call the set of such potential functions $\mathcal{E}(\alpha,\gamma,m)$.
\end{enumerate}

Assumption~\ref{itm:assumptionA}  is a relaxation of assumptions in several past papers on Langevin algorithms where the potential $f$ was considered to be both  $m$-strongly convex ~\eqref{eq:assumption_scparam} and $L$-smooth~\eqref{eq:assumption_smoothness}(~\cite{dalalyan2016theoretical,dwivedi2018log,cheng2017convergence}). We note that for the class of smooth potentials the target distribution satisfies the conditions of Assumption~\ref{itm:assumptionA}. In addition, it is worth noting that Assumption~\ref{itm:assumptionB} demands tighter concentration on a region with high probability\footnote{ $\propto$ here signifies that the diameter is proportional to the dimension  up to a polylogarithmic constant independent of the dimension}, it is not difficult to prove that Assumption~\ref{itm:assumptionA} is implied by Assumption~\ref{itm:assumptionB}.


\subsection{Metropolized Random Walk Algorithms}
~\label{sub:related_sampling_algorithms}
We describe here the general structure of said Markov chains given an initial density $\mu_0$ over $\mathbb{R}^d$, which is simulated in two steps: 

\begin{enumerate}
    \item \textbf{Proposal Step}: Given a \textit{proposal function} $Q: \mathbb{R}^d\times\mathbb{R}^d\rightarrow \mathbb{R}_{+}$, where $Q(x,.)$ is a proposal density for $x\in \mathbb{R}^d$. At each iteration, and given a current state $x\in\mathbb{R}^d$ of the chain, a new vector $z\in\mathbb{R}^d$ is proposed by sampling from $Q(x,.)$.
    \item \textbf{Accept-Reject correction Step}: The proposed sample $z\in\mathbb{R}^d$ is retained   as the new state of the  Chain with probability: 
    \begin{equation}
        A(x,z):=\min \left\{1,\frac{\pi(z)Q(z,x)}{\pi(x)Q(x,z)}\right\}
    \end{equation}
    Otherwise it is rejected and the chain remains at $x\in\mathbb{R}^d$ with probability $1- A(x,z)$.
\end{enumerate}
This construction yields a reversible Markov chain with stationary distribution $\pi$ and a transition kernel given by
\begin{subequations}
\begin{align}
    T(x,y):=T_x(y)= [1-A(x)]\delta_x(y)+ Q(x,y)A(x,y)\quad \\
     A(x)=\displaystyle\int Q(x,y)A(x,y)dy=\displaystyle\int Q(x,y)\min \left\{1,a(x,y)\right\}dy=\min \left\{1,\frac{\pi(y)Q(y,x)}{\pi(x)Q(x,y)}\right\}
\end{align}
\end{subequations}


\paragraph{Transition operator.} 
\label{par:transition_operator_}

Let $\mathcal{T}$ be the transition operator of the Markov
chain on 
$\mathcal{X}$. That is, given an initial distribution $\mu_0$ on $\mathcal{X}$,  $\mathcal{T}(\mu_0)$ stands for the
distribution of the next state of the chain. We have for any measurable set $A \in \mathcal{B}(\mathcal{X})$: 
$\mathcal{T}(\mu_0)(A) =\int_{\mathcal{X}} T (x, A)
\mu_0(x) dx$ 
Similarly, $\mathcal{T}^k$ denotes the $k$-step transition operator.
 $\mathcal{T}_x$ denotes $\mathcal{T}(\delta_x)$,
the \emph{transition distribution at $x$}, where $\delta_x$ is the 
Dirac distribution at $x \in \mathcal{X}$.
We have that $\mathcal{T}_x=\mathcal{T}(x, \cdot)$.

\paragraph{Distances between distributions.}
\label{par:convergence_criterion_}

One tractable way to characterize the convergence of a Markov chain is to consider its mixing time with respect to a given class of distances, one such example is the class of $\mathcal{L}_p$-distances for $p \geq
1$. For a given distribution $Q$  with density $q$, we define its 
$\mathcal{L}_p$-divergence with respect to the positive density $\nu$ as
the following : 
\begin{subequations}
\label{eq:convergence}
\begin{align}
\label{eq:l_p_distance}
 d_{\mathfrak{p}}(Q, \nu) = \left(\int_{\mathcal{X}}
   \left|\frac{q(x)}{\nu(x)} - 1\right|^{\mathfrak{p}} \nu(x) dx \right)^{\frac{1}{\mathfrak{p}}}.
\end{align}

Note that when $\mathfrak{p}=2$, we recover the  $\chi^2$-divergence. For $\mathfrak{p}= 1$,
the distance $d_{1}(Q, \nu)$ is equivalent to two times the total variation
distance between $Q$ and $\nu$. Where  $d_{\text{TV}}(Q,\nu)$ denotes the total variation distance.

\paragraph{Mixing time of a Markov chain.}
\label{par:mixing_time_of_markov_chains_}

The mixing time of a Markov chain is defined as the minimum number of steps taken by the chain in order to be within $\epsilon$  of the target distribution in $\mathcal{L}_p$-norm, given the fact that it starts from an initial distribution $\mu_0$. More formally, given a Markov chain with initial distribution $\mu_0$, a transition operator $\mathcal{T}$ and a target distribution $\Pi$  with density $\pi$. Its $\mathcal{L}_p$ mixing time with respect to $\Pi$ is given by
\begin{align}
\label{eq:tmix_defn}
  \tau_{\mathfrak{p}}(\epsilon; \mu_0) = \inf\left\{{k \in \mathbb{N}
     | \quad  d_{\mathfrak{p}}\left({\mathcal{T}^k(\mu_0), \Pi}\right) \leq
    \epsilon}\right\}.
\end{align}
where $\epsilon > 0$ is a predefined error tolerance.
\end{subequations}


\paragraph{Warmness of initial distribution.}
\label{par:initial_distribution_}

Next, we define the notion of warmness. More formally, a Markov chain on  $\mathcal{X}$ with stationary distribution $\Pi$ is said to have a $\beta$-\textit{warm start} if its initial distribution $\mu_0$ is such that
\begin{align}
\label{eq:def_warm}
\sup_{A \in \sigma(\mathcal{X})}\left( \frac{\mu_0(A)}{\Pi(A)}\right) \leq
\beta,
\end{align}
where $\sigma(\mathcal{X})$ stands for the Borel $\sigma$-algebra of $\mathcal{X}$. In simpler terms, $\mu_0$ is said to be a warm start if $\beta$ is a small constant independent of the dimension $d$.
\paragraph{Notable Metropolized random walk Algorithms.} 

We describe here some of the most celebrated sampling 
algorithms on $\mathcal{X}=\mathbb{R}^d$. First, we consider the Metropolized Random Walk algorithm (RWM),
next we consider the Metropolis adjusted Langevin algorithm (MALA).

\subsubsection{Metropolized Random Walk algorithm (RWM)} 

The Metropolized Random Walk algorithm (RWM) considers the task of constructing a Markov chain in order to sample from  densities taking the form~\eqref{EqnDefnTarget} defined on $\mathbb{R}^d$. Indeed, Given a state $x_t \in\mathbb{R}^d$ at iterate $t$,  the algorithm proceeds by generating a proposal vector $z_{t + 1} \sim
\mathcal{N}(x_t, 2h\mathbb{I}_d)$, where $h > 0$ is a step-size
parameter specified. The algorithm decides to accept or reject $z_{t+1}$
using a Metropolis-Hastings decision rule, see Algorithm~\ref{algo:mrw} for more details. The RMW algorithm is a zeroth order method since it requires information about the function $f$ only by accessing function values, not its gradients.

\begin{algorithm}[ht]
\footnotesize
  \KwIn{Step size $h> 0$ and a sample $x_0$ from a starting
    distribution $\mu_0$} \KwOut{Sequence $x_1, x_2,\ldots$}
  \For{$t=0, 1, \ldots $}
  {%
      \textbf{Proposal step}: \emph{Draw} $z_{t+1} \sim \mathcal{N}(x_t, 2h
      \mathbb{I}_d)$  \\
       \textbf{Accept-reject step}:\\
      \quad compute $A_{t+1} \gets \displaystyle \min \left\{1,
          \frac{\exp(-f(z_{t+1}))}{\exp(-f(x_{t})}\right\}$\\
        \quad With probability $A_{t+1}$ \emph{accept}  proposal:
        $x_{t+1} \gets z_{t+1} $\\
        \quad With probability $1-A_{t+1}$ \emph{reject}  proposal:
        $x_{t+1} \gets x_t$
  }
  \caption{Metropolized Random Walk (RWM)}
  \label{algo:mrw}
\end{algorithm}
\subsubsection{MALA algorithm}
The Metropolis-adjusted Langevin algorithm (MALA) on the other hand is a first order method, it requires access to the function value
$f(\cdot)$, aswell as its gradient $\nabla f(\cdot)$
at any state $x \in \mathbb{R}^d$.  Given state $x_t$ at iterate $t$,
we query the value of $\left(f(x_t), \nabla f(x_t)\right)$ and then proceeds by generating  a new proposal $z_{t+1} \sim \mathcal{N}(\displaystyle x_{t} - h\nabla f(x_{t}), 2h \mathbb{I}_d)$, followed by
Metropolis-Hastings correction;  we refer to Algorithm~\ref{algo:mala} for the details. The MALA algorithm can be deduced Langevin diffusion (SDE) by considering an Euler-Maruyama discretization scheme. The
Langevin diffusion being a stochastic  whose stochastic differential equation (SDE) is given by: 

\begin{align}
  \label{eq:Langevin_SDE}
  dX_t = -\nabla f(X_t) dt + \sqrt{2} dW_t.
\end{align}

\begin{algorithm}[ht]
\footnotesize
  \KwIn{Step size $h >0$ and a sample $x_0$ from a starting
    distribution $\mu_0$} \KwOut{Sequence $x_1, x_2,\ldots$}
  \For{$i=0, 1, \ldots $}
    {%
      \textbf{Proposal step}: \emph{Draw} $z_{t+1} \sim \mathcal{N}(\displaystyle
      x_{t} - h \nabla f(x_{t}), 2h
      \mathbb{I}_d)$  \\
       \textbf{Accept-reject step}:\\
      \quad compute $A_{t+1} \gets \displaystyle\min \left\{1,\frac{\exp\left(-f(z_{t+1})-\left\|x_t-z_{t+1}+h\nabla f(z_{t+1})\right\|^2/4h\right)}{\exp\left(-f(x_{t})-\left\|z_{t+1}-x_{t}+h\nabla f(x_{t})\right\|^2/4h\right)}
          \right\}$\\
        \quad With probability $A_{t+1}$ \emph{accept} proposal:
        $x_{t+1} \gets z_{t+1} $\\
        \quad With probability $1-A_{t+1}$ \emph{reject} proposal:
        $x_{t+1} \gets x_t$
  }
  \caption{Metropolis adjusted Langevin algorithm (MALA)}
  \label{algo:mala}
\end{algorithm}

We can see that the proposal of the MALA algorithm is a Gaussian distribution that is centered at one gradient descent step starting from the current position. The convergence of gradient descent is usually shown using smoothness of the function $f$, and gradient descent can diverge for any fixed step size $h$ for functions that do not satisfy this condition. For example, if $f(x)=\|x\|^2+\|x\|^4$, then it is easy to see that for any $h>0$, the average acceptance rate of one MALA step starting from $x$ tends to zero as $\|x\|\to \infty$. This means that MALA is not well suited for such potentials as it may get stuck in the tails. As we shall see in the following section, our new algorithm is able to work well even for such challenging thin tailed potentials.

\section{ MAO Algorithm}
\label{sec:New_algo}
In this section, we describe our novel Metropolized algorithm. The main idea is to construct a novel proposal based on an auto-regressive kernel. To that end, we require access to the mode of the distribution $\pi$, which we can obtain, at least approximately, by running an offline Optimization algorithm. We then make use of the accept-reject step to correct with respect to the target distribution, this method has its merit in the fact that one hand, it is fairly cheap to sample from a Gaussian distribution. On the other hand, it allows us to leverage the extensive optimization toolkit ~\cite{Bubeck1}, and ~\cite{nesterov} to obtain rates of convergence for the learning of the mode; see Algorithm~\ref{algo:new} for the details.

Figure~\ref{fig:target}\footnote{The red net is the proposal distribution centered at the mode of the target distribution.} provides visual illustration as to the idea of the Gaussian approximation.
\begin{figure}
    \centering
    \includegraphics[width=.6\linewidth]{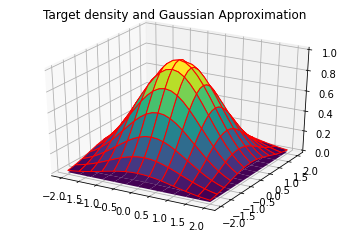}
    \caption{Target density $\pi(x)\propto e^{-\left\|x\right\|^4-\left\|x\right\|^2}$ and Gaussian approximation (red net) centered at mode $\tilde{x}$.}
    \label{fig:target}
\end{figure}
\begin{algorithm}[ht]
\footnotesize

  \KwIn{Step size $h >0$ and a sample $x_0$ from a starting
    distribution $\mu_0$, $\tilde{x}$: mode approximation} \KwOut{Sequence $x_1, x_2,\ldots$}
  \For{$i=0, 1, \ldots $}
    {%
      \textbf{Proposal step}: \emph{Draw} $z_{t+1} \sim \mathcal{N}(\displaystyle
      x_{t} - h(x_t - \tilde{x}), 2h
      \mathbb{I}_d)$  \\
       \textbf{Accept-reject step}:\\
      \quad compute $A_{t+1} \gets \displaystyle\min \left\{1,\frac{\exp\left(-f(z_{t+1})-\left\|x_t-z_{t+1}+h(z_{t+1}-\tilde{x})\right\|^2/4h\right)}{\exp\left(-f(x_{t})-\left\|z_{t+1}-x_{t}+h(x_{t}-\tilde{x})\right\|^2/4h\right)}
          \right\}$\\
        \quad With probability $A_{t+1}$ \emph{accept} the proposal:
        $x_{t+1} \gets z_{t+1} $\\
        \quad With probability $1-A_{t+1}$ \emph{reject} the proposal:
        $x_{t+1} \gets x_t$
  }
  \caption{Metropolized Algorithm with Optimization Step (MAO)}
 \label{algo:new}
\end{algorithm}

\section{Main Results}
\label{sec:main_results}
In this section we state our main results.  We remind the reader that MAO refers to our  Metropolized Algorithm with Optimization step. First, we state our results for (MAO):  we derive the mixing time bounds for target distributions satisfying  Assumption~\ref{itm:assumptionA} in Theorem~\ref{thm_a}, then we state mixing time bounds for target distributions satisfying  Assumption~\ref{itm:assumptionB} in Theorem~\ref{thm_b}. 

\subsection{Mixing time bounds for MAO}
\label{sub:mixing_time_bounds_for_MAO}

We state mixing time bounds for MAO for the class of target
distributions $\Pi$ satisfying Assumption~\ref{itm:assumptionA}.
Let MAO-($\tilde{x},h$)~denote the $\frac{1}{2}$-lazy MAO algorithm with
step size $h$ and $\tilde{x}$-mode output\footnote{In this section we assume access to the exact mode of the distribution at the Output of the Offline Optimization Scheme, we will consider the case of $\delta$-error Later in Proof Section}. Let $\tau_2^{MAO}(\epsilon; \mu_0)$
denote its $\mathcal{L}_2$-mixing time~\eqref{eq:tmix_defn} when starting from an initial distribution $\mu_{0}$.
$c$ here denotes a universal constant.
We consider an MAO-chain targeting 
distributions satisfying the assumptions of Assumption~\ref{itm:assumptionA}. 

We  state an explicit mixing time bound of MAO. We consider an $(\alpha,m)$--\textit{strongly
log-concave} target distribution $\Pi$
(assumption~\ref{itm:assumptionA}). In the statement of our results, we make use of the standard function appearing in several past works~\cite{chenhmc,dwivedi2018log}: 
\begin{subequations}
\begin{align}
\nonumber 
    R(s) &:= 1 + \max \left\{  \left(\frac{\log(1/s)}{d}
    \right)^{1/4},  \left(\frac{\log(1/s)}{d}\right)^{1/2} \right\},\\
\label{eq:def_radius}
    r(s)&:=\max\left(R(s), R\left(\frac{1}{2}\right)\right),
\end{align}
\end{subequations}
for $s >0$, and involves the step-size choice\footnote{$c'$ here donotes  a universal constant } 
\begin{align*}
    h_{warm}=\frac{1}{c_A r\left(\frac{\epsilon^2}{3\beta}\right)d^{\alpha-1}}.
\end{align*}
Given the definitions, we can state the following. 
\begin{thm}
\label{thm_a}
For  $(m,\alpha)$--\textit{strongly
log-concave} target distribution (cf. Assumption~\ref{itm:assumptionA}) and a $\beta$-warm initial distribution $\mu_0$. We have for any error tolerance
  $\epsilon \in (0, 1)$, and hyperparameter $h_{\text{warm}}$, the MAO-($h$,$\mu_0$) chain satisfies: 
    \begin{align}
    \tau_2^{MAO}(\epsilon; \mu_0) \leq c_A r\left(\frac{\epsilon^2}{3\beta}\right)d^{\alpha-1}\log\left({\frac{\log(\beta)}{\epsilon}}\right)
   \end{align}
\end{thm}

Next we state results for the mixing time bound for MAO chains targeting a 
distribution $\Pi$ satisfying Assumption~\ref{itm:assumptionB}. Theorem \ref{thm_b} states the results for distributions satisfying the conditions of Assumption~\ref{itm:assumptionB} under a $\beta$-warm start, for a choice of hyper-parameter $h$
\begin{align}
    h_{warm}=\frac{1}{c_B \tau\left(\frac{\epsilon^2}{3\beta}\right)d^{\omega}} \quad \textrm{with} \quad \omega=\max
\left(\frac{2 (\alpha-1)}{\gamma},\frac{\gamma+\alpha-2}{\gamma}\right).
\label{eq:hwarmomega}
\end{align}
\begin{thm}
  \label{thm_b}
  Consider an $(\alpha,\gamma,m)$--\textit{strongly
    log-concave} target distribution (cf. Assumption~\ref{itm:assumptionB}) and a $\beta$-warm initial distribution $\mu_0$. Then for any
   accuracy $\epsilon \in (0, 1)$, for step size $h_{warm}$ defined as in \eqref{eq:hwarmomega}, MAO with $\beta$-warm initial distribution $\mu_0$ start satisfies 
   \begin{align}
    \tau_2^{MAO}(\epsilon; \mu_0) \leq c_B  
    \tau\left(\frac{\epsilon^2}{3\beta}\right)d^{\omega}\log\left({\frac{\log(\beta)}{\epsilon}}\right)
   \end{align}
\end{thm}
\paragraph{Discussion of bounds from warm start.} 
Theorem \ref{thm_a} provides mixing time bounds for MAO for target  distributions satisfying the conditions detailed in Assumption~\ref{itm:assumptionA}, while Theorem~\ref{thm_b} provides mixing time bounds for MAO for target  distributions satisfying the conditions detailed in Assumption~\ref{itm:assumptionB}. Theorem \ref{thm_a} implies that given a $\beta$-warm start for a  $(\alpha,m)$--\textit{strongly log-concave} target distribution $\epsilon$-$\mathcal{L}_{2}$ mixing time \footnote{Note that for a large range of values of $\epsilon$ we can treat $r$ as a small constant } of MAO scales as $\tilde{\mathcal{O}}\left(d^{\alpha-1} \log(\frac{\log{\beta}}{\epsilon})\right)$. Theorem~\ref{thm_b} provides mixing time bounds for MAO for target  distributions satisfying the conditions detailed in Assumption~\ref{itm:assumptionB} scaling as $\tilde{\mathcal{O}}\left(d^{\omega}
\log(\frac{\log{\beta}}{\epsilon})\right)$ for $\omega=\max
\left(\frac{2 (\alpha-1)}{\gamma},\frac{\gamma+\alpha-2}{\gamma}\right)$. The existing convergence results for MALA are not applicable for thin tailed distributions.
\paragraph{Gaussian Case.} 
Let us consider the case where the  gradient of the potential $\nabla f$ is globally liptschitz, which corresponds to the special case where $\alpha=1$ and $\gamma=2$, which yields an $\epsilon$-$\mathcal{L}_{2}$ mixing time that scales as $\tilde{\mathcal{O}}\left(d\log(\frac{\log{\beta}}{\epsilon})\right)$, thus we confirm the improved mixing time rate derived in the ~\cite{chenhmc} for MALA which is an improvement on the rates obtained by ~\cite{dwivedi2018log} by improving the dependency on the warmness. The novelty of our work lies in the enhanced rates for thin tailed models, on which MALA potentially fails (\cite{roberts1996exponential}) to converge we refer to the numerical simulation section for further details. In the next session we reproduce the main framework allowing us to enhance the dependency on the warmness, as introduced in the work of ~\cite{chenhmc}.

%

\subsection{Controlling Optimization Error for MAO}
\label{sub:opti_error_mao}
As discussed earlier, the bounds on mixing times derived in Theorem~\ref{thm_a} and Theorem~\ref{thm_b}, were stated in a setting where the exact value of $x^*$-the mode of the target distribution $\Pi$ was known. However in practice it is rarely the case, seeing as we often run an offline optimization algorithm in order to obtain an approximation $\tilde{x}$ of the mode up to a certain error threshold $\delta$, the task then becomes that of controlling the induced error of running the offline optimization algorithm on the mixing time of MAO. Theorem \ref{thm:opti_mao} provides results as to the effect of the error induced by running an off-line optimization scheme.
\paragraph{Problem Set up.}
Our optimization problem of interest in this case can be formulated as the following :
\begin{align*}
    x^* := \text{arg}\max_{\text{s.t}\  x \in \Omega} \pi(x)= \text{arg}\max_{\text{s.t} \ x \in\Omega} e^{-f(x)}=\text{arg}\min_{\text{s.t} \ x \in \Omega} f(x)
\end{align*}

We then run an optimization algorithm, we refer to the large body of literature on first order optimization algorithm suited for our setup see~\cite{Lu} and~\cite{Maddison} for further details on optimization algorithms, at which point we obtain as an output $\tilde{x}$ such that 
\begin{align*}
    \left\|\tilde{x}-x^*\right\|\leq \delta
\end{align*}

Keeping in line with our assumptions we announce the following theorem: 
\begin{thm}
  \label{thm:opti_mao}
   Consider a target distribution $\Pi$ satisfying Assumption \ref{itm:assumptionA} or Assumption \ref{itm:assumptionB}. Then for any 
   accuracy $\epsilon \in (0, 1)$, and a $\beta$-warm distribution $\mu_0$, and for $\delta$-$\tilde{x}$ output of an optimization algorithm, then there exists a choice of the accuracy $\delta = \mathcal{O}(\frac{1}{h}\log(\frac{1}{\epsilon}))$ such that for the MAO-($\tilde{x},\delta,\mu_0$)  chain satisfies the bounds of Theorems \ref{thm_a} and \ref{thm_b} with different constants $c_A'$ and $c_B'$ in the place of $c_A$ and $c_B$.
\end{thm}
The key insight of Theorem \ref{thm:opti_mao} is that guarantees a choice of $\delta$ allowing us to match the mixing time for MAO given in Theorem \ref{thm_a} and Theorem \ref{thm_b} for the case in which direct access to the mode of the target distribution is available, this highlights another aspect of the merit of MAO over MALA. Indeed, in certain regimes it might be computationally cheaper to run an optimization algorithm in instances where computing the gradient is prohibitive. We refer to the Proofs section for a formal proof of the result. In the next section, we give numerical illustration to the results stated above.

\paragraph{Note on Optimization Schemes.}
MAO assumes that we have found a mode approximation $\tilde{x}$, i.e. the minimizer of $f$, using an optimization method. In the case of thin tailed potentials, standard smoothness assumptions such as $\|\nabla^2 f(x)\|\le C$ may not be applicable, so some optimization methods such as gradient descent may diverge. Nevertheless, the recent paper of \cite{Lu} has proposed an algorithm that is able to work even for potentials with thin tails.
\begin{algorithm}[!ht]
\caption{Primal Gradient Scheme with reference function $h(\cdot)$}\label{proxgradscheme}
\begin{algorithmic}[1]
\STATE {\bf Initialize.}  Initialize with $x^0 \in Q$.  Let $L$, $h$ satisfying Definition(1)~\cite{Lu} of  be given. \\
At iteration $i$ :\\
\STATE  {\bf Perform Updates.}  Compute $\nabla f(x^i)$ , \\\medskip
\ \ \ \ \ \ \ \ \ \ \ \ \ \ \ \ \ \ \ \ \ \ \ \ \ \ \ \ $x^{i+1} \gets \arg\min_{x \in Q} \{f(x^i) + \langle \nabla f(x^i), x-x^i \rangle + L D_h(x,x^i) \}$ .\\\medskip
\STATE  {\bf Return}  mode approximation $\tilde{x}$ , \\\medskip
\end{algorithmic}
\end{algorithm}\medskip
Here $D_h(x,x^i)=h(x^{i})-h(x)-\left<\nabla h(x^i),x-x^{i}\right>$ is the so called Bregman divergence.
Note that each iteration requires the solution of a sub-problem related to the function $h$. By choosing $h(x):=\frac{\|x\|^2}{2}+\frac{\|x\|^{\alpha}}{\alpha}$, it is easy to check that the so-called relative smoothness condition holds for every $f\in \epsilon(\alpha)$, hence convergence to the minimizer is guaranteed. Another possible approach is to use the method proposed in \cite{Maddison}.
\subsection{Feasible starts}
In contrast to previous works~\cite{chenhmc},~\cite{dwivedi2018log} where Gaussian starts could be considered as feasible starts, due to the nature of the $\mathcal{E}(\alpha)$ class of thin tailed distributions such distributions cannot be considered as feasible starts, since a feasible start needs to to have tails as thin or thinner than the target distribution.

Suppose that $f\in \mathcal{E}(\alpha)$, and that conditions \eqref{eq:alpha_f} and \eqref{eq:nabla_f} hold with constants $K_1$ and $K_2$. 
Proposition \ref{prop:feasiblestart} below shows that a starting distribution with potential
\begin{equation}\label{eq:f0def}
f_{0}(x):=K_2 \left(\frac{\|x\|^2}{2}+\frac{\|x\|^{\alpha}}{\alpha(\alpha-1)}\right)
\end{equation}
will be $\beta$-warm with respect to the target for an appropriate choice of $\beta$. Sampling from such an isotropic distribution can be efficiently carried out by first sampling the radius, and then the direction vector uniformly. The proof of the proposition is included in Section \ref{sec:proof_of_feasiblestart}.

\begin{prop}\label{prop:feasiblestart}
Let $f\in \mathcal{E}(\alpha)$, $f_0$ be defined as in \eqref{eq:f0def}, and $\mu_0(x)\propto e^{-f_0(x)}$. Suppose that $f$ is minimized at $0$ (this can be also achieved by shifting $f_0$ to have the same minimum as $f$). Suppose that $f$ is $m$-strongly convex. Then $\mu_0$ is $\beta$-warm with respect to $\Pi$, with $\beta$ satisfying that
\[\log(\beta)= \frac{d}{2} \log\left(\frac{2}{m}\right) +\log\left(\Gamma\left(\frac{d}{2}+1\right)\right)-\log\left(\Gamma\left(\frac{d}{\alpha}\right)\right)-\log(d)+\log(\alpha)+\frac{K_2}{2}+\frac{d}{\alpha}\log\left(\frac{2K_2}{\alpha(\alpha-1)}\right).\]
\end{prop}
Since $\log(\Gamma(x))\propto x \log(x)$ by Stirling's formula, this means that $\log(\beta)=\tilde{\mathcal{O}}(d \log(d))$ as $d\to \infty$.



\section{Numerical experiments}
\label{sec:numerical_experiments}
In this section, we conduct several numerical experiments in order to inspect the performance of  MAO and compare it with MALA with the aim of illustrating the validity of our claims in Theorems [\ref{thm_a} ,\ref{thm_b}], and the mixing time results for MAO. Our theoretical results suggests that MAO should outperform MALA on the class of $\mathcal{E}(\alpha,m)$ distributions. 

In our  simulations, we inspect the dimension $d$ dependency and
condition number $\alpha$ dependency for $\mathcal{E}(\alpha,m)$ class of distributions
guided by our step-size choices. At first, we consider the task of  sampling from a target distribution with density :
\begin{align}
    \Pi_1(x)\propto e^{-\frac{\left\|x\right\|^4}{4}-\frac{a\left\|x\right\|^2}{2}}
\end{align}
Where $a$ is a constant controlling the magnitude of the Gaussian perturbation. The log density along with its derivatives are given by:
\begin{align}
    f(x) = \frac{\left\|x\right\|^4}{4}+\frac{a\left\|x\right\|^2}{2},\quad \nabla f(x) = \left\|x\right\|^3\cdot x +ax, \quad\text{and}\quad \nabla^2 f(x) = a\mathbb{I}_d+ 3 D(x).
\end{align}
Where $D(x)=\mathbf{diag}(x^2)$.

We deduce that  the potential $f$ is indeed strongly convex with a parameter $m=a$, and as highlighted in the proof section it  also satisfies the conditions of Assumption~\ref{itm:assumptionB}.

Next we consider the case where the tails are thin are in certain directions and decay faster than exponential in other directions. More formally, we consider the task of sampling from a  target distribution equipped with a density :
\begin{align}
    \Pi_2(x)\propto e^{-\frac{\left\|x\right\|^4}{4}-\frac{x_{1}^2}{2}}
\end{align}
In which case the log density along with its derivatives are given by:
\begin{align}
    f(x) = \frac{\left\|x\right\|^4}{4}+\frac{x_{1}^2}{2},\quad \nabla f(x) = \left\|x\right\|^3\cdot x +x_1\cdot e_1, \quad\text{and}\quad \nabla^2 f(x) = E_{1,1}+ 3 D(x).
\end{align}
Where $D(x)=\mathbf{diag}(x^2)$ and $E_{1,1}=e_1e_1^{\intercal}$.

Figures [\ref{fig:target1},\ref{fig:target2}] shows the contour plots of the densities of the targets considered in our numerical experiments.

\begin{figure}[!ht]
\centering
\begin{minipage}{.5\textwidth}
  \centering
  \includegraphics[width=.9\linewidth]{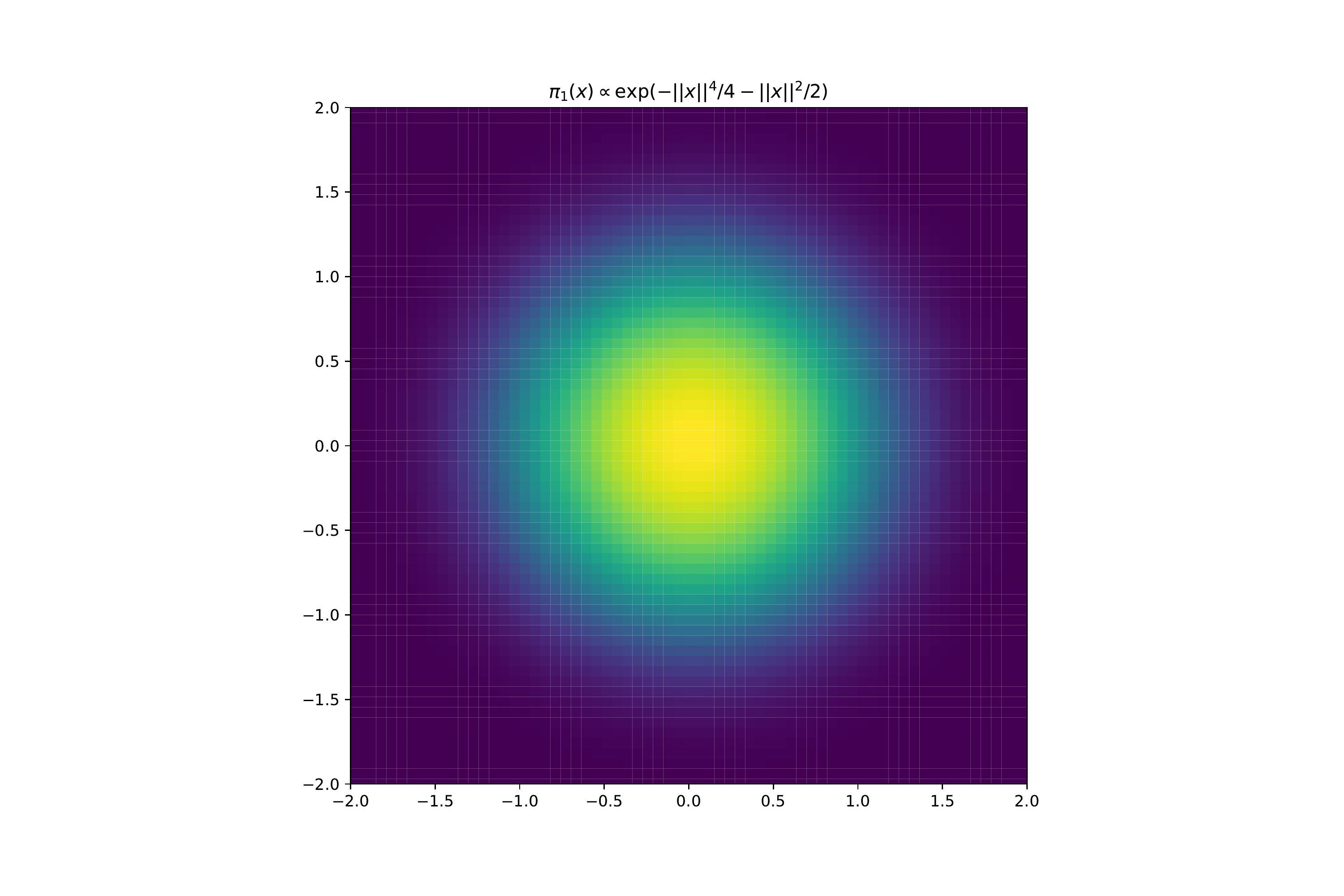}
  \caption{Contour plot of $\pi_1(x)\propto e^{-\frac{\left\|x\right\|^4}{4}-\frac{\left\|x\right\|^2}{2}}$}
    \label{fig:target1}
\end{minipage}%
\begin{minipage}{.5\textwidth}
 \centering
  \includegraphics[width=.9\linewidth]{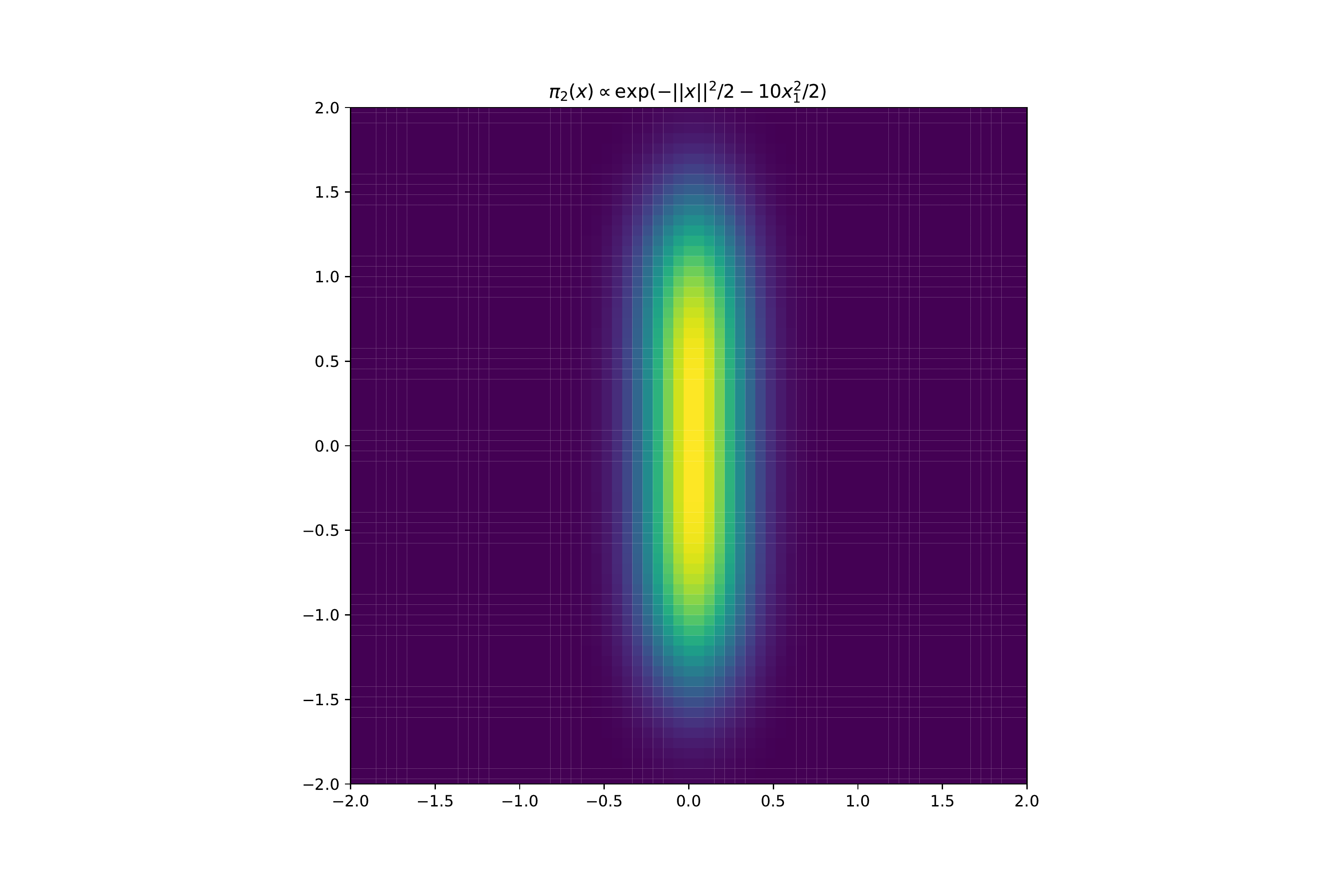}
   \caption{Contour plot of $\pi_2(x)\propto e^{-\frac{\left\|x\right\|^4}{4}-\frac{10 x_1^2}{2}}$}
    \label{fig:target2}
\end{minipage}
\end{figure}
We consider two different experiments, first sampling from the target distribution $\Pi_1$ and sampling from the target distribution $\Pi_2$. The step-size choice of MALA used in our simulations is the same as the one retained for MAO  given by the results of Theorem~\ref{thm_b}. We then inspect different measures for convergence diagnostics, namely, traceplots along other coordinates and auto-correlation plots, and average acceptance probability and effective sample size as the dimension increases. In particular, these last two measures help diagnose the Markov chain's convergence rate in a single long run, allowing us to inspect the merit of MAO over MALA, thus supporting our results.

First, we run the MALA algorithm Targetting $\Pi_1$ for 1 000 000 iteration with a burn-in period set at 100 000, Figure~\ref{fig:mala_d2}  shows the traceplots of MALA along the first coordinate when the dimension of the problem is fixed at $d=2$ initially.

\begin{figure}[!htp]
\centering
\begin{minipage}{.5\textwidth}
  \centering
  \includegraphics[width=\linewidth]{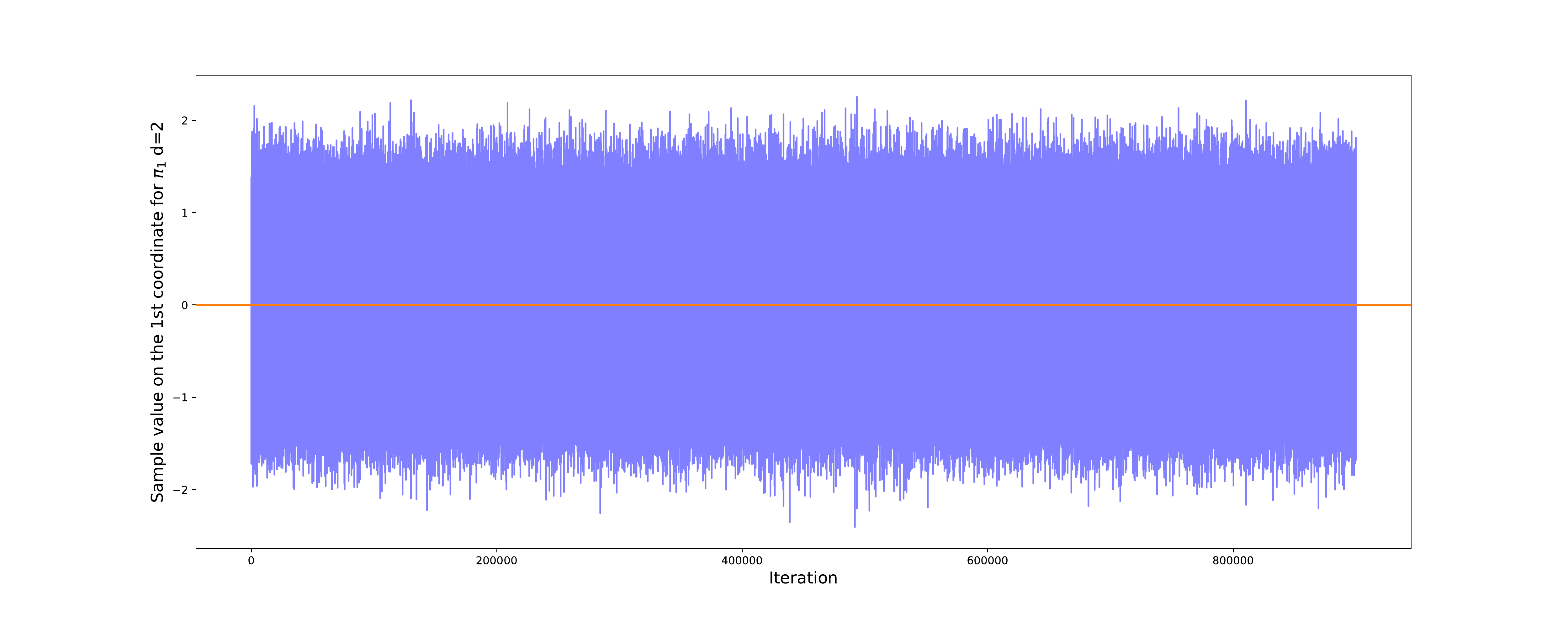}
  \caption{Traceplot on $x_1$ MALA chain targeting $\Pi_1$}
    \label{fig:mala_d2}
\end{minipage}%
\begin{minipage}{.5\textwidth}
  \centering
  \includegraphics[width=.7\linewidth]{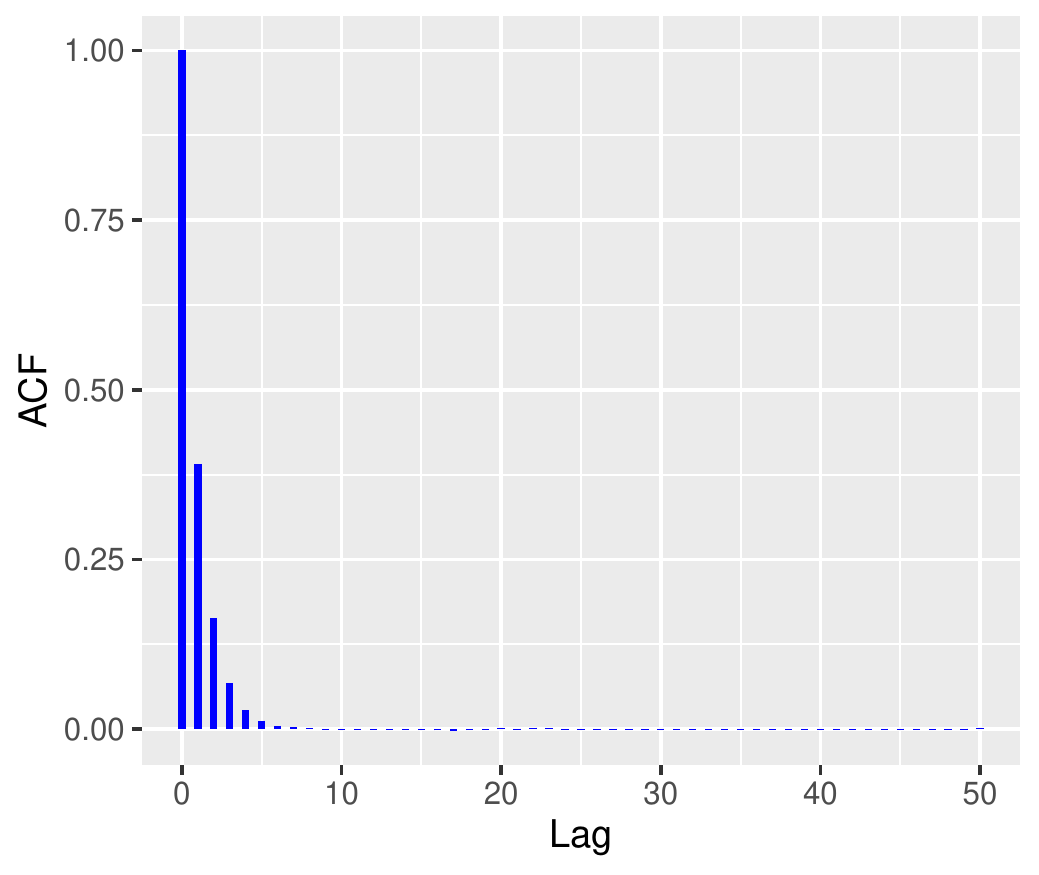}
   \caption{auto-correlation plot MALA chain targeting $\Pi_1$}
    \label{fig:auto_mala_d2}
\end{minipage}
\end{figure}
\
Figure[\ref{fig:auto_mala_d2}], on the other hand, shows the auto-correlation plot for the same chain. We can read from Table [\ref{tab:mala_pi1}] that the expected acceptance rate is around $0.827$ when targeting $\Pi_1$ and considering the effective sample size (ESS) along the two coordinates, we obtain a satisfactory convergence. Next, we inspect the scaling of the convergence of the MALA chain with the dimension of the problem $d$. To that end, we consider running the MALA chain targeting $\Pi_1$ on a grid of dimensions. Thus, allowing us to inspect the mixing of the MALA chain in a high dimensional setting.

Figures [\ref{fig:mala_d64},\ref{fig:auto_mala_d64}] show the traceplot along the first coordinate along with  the auto-correlation plot when the dimension of the problem is set at $d=64$. Inspecting these two figures, we gain several insights. We can read from Table [\ref{tab:mala_pi1}] that the expected acceptance rate is around $0.902$, we notice that the traceplot fails to stabilize. The auto-correlation plot takes considerable time to reach $0$, both signs of bad performance in terms of convergence.  

\begin{figure}[!htp]
\centering
\begin{minipage}{.5\textwidth}
  \centering
  \includegraphics[width=\linewidth]{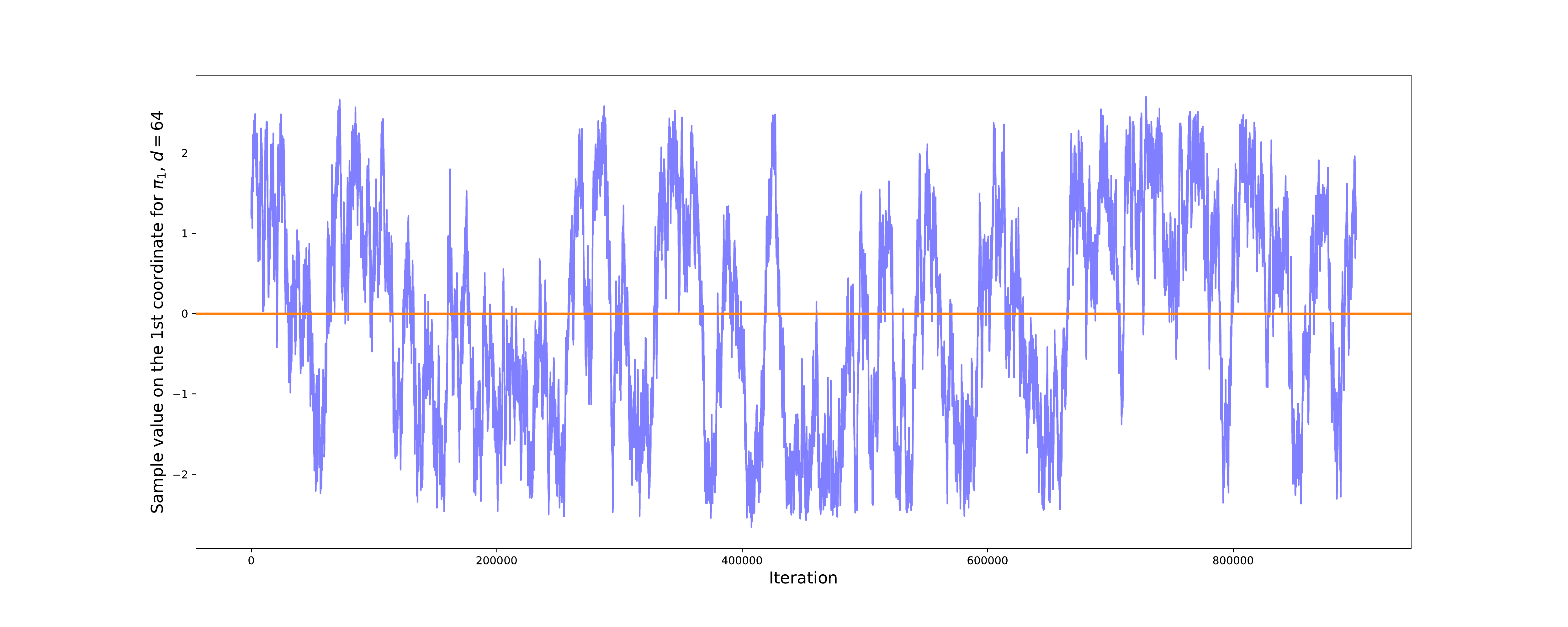}
  \caption{Traceplot on $x_1$ MALA chain targeting $\Pi_1$}
    \label{fig:mala_d64}
\end{minipage}%
\begin{minipage}{.5\textwidth}
  \centering
  \includegraphics[width=.7\linewidth]{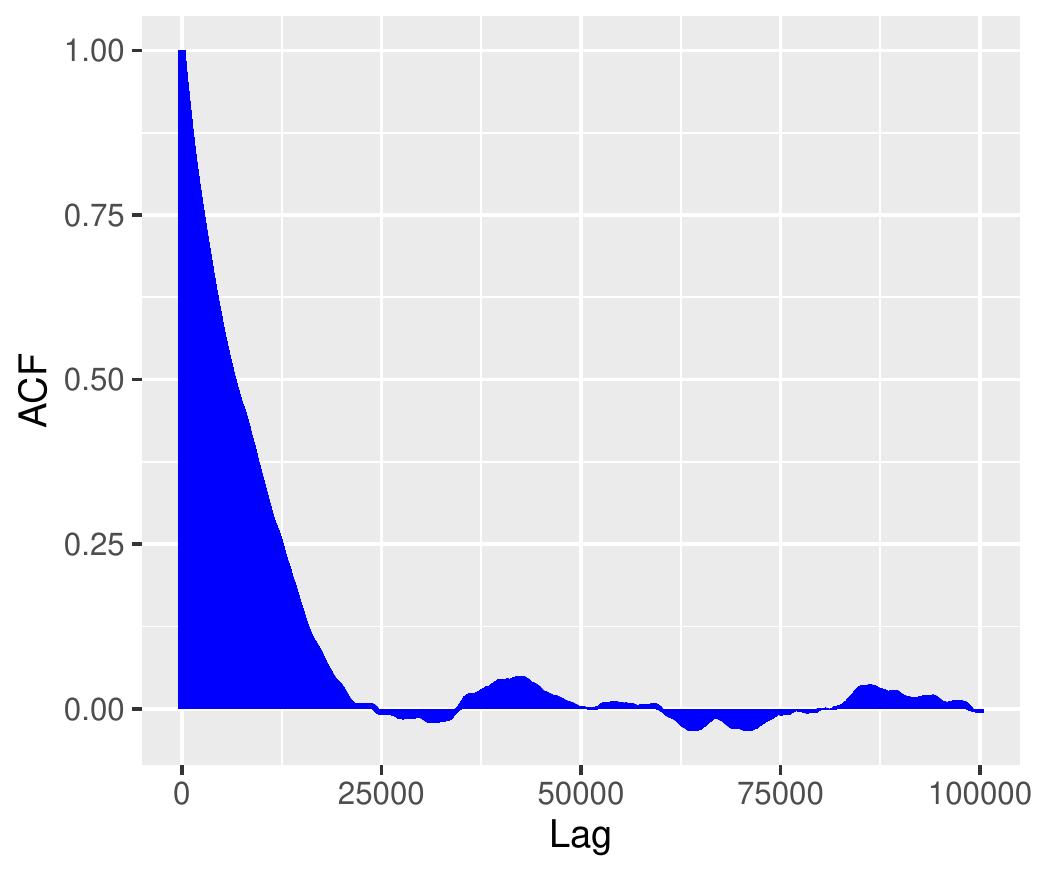}
   \caption{auto-correlation plot MALA chain targeting $\Pi_1$}
    \label{fig:auto_mala_d64}
\end{minipage}
\end{figure}

In order to further  examine  the  scaling  of the performance of MALA on the $\mathcal{E}(\alpha)$ class we consider the task of sampling from $\Pi_2$, where the tails along the coordinates other than $x_1$ decay with a rate faster than Gaussian tails, in light of the work of ~\cite{roberts1996exponential} we expect the performance of MALA to deteriorate, this intuition is further confirmed when inspecting Figures [\ref{fig:mala_d2_pi2_1},\ref{fig:mala_d2_pi2_2}] where the traceplots highlights the failure of MALA to stabilize, and the performance deteriorates noticeably when the dimension increases as highlighted by Figures [\ref{fig:mala_d64_pi2_1},\ref{fig:mala_d64_pi2_2}]. Another indication on the sub optimal performance of MALA when targeting $\Pi_2$ is further highlighted by the convergence diagnostics given by Table[\ref{tab:mala_pi2}] where we notice that the (ESS) falls from 4293 on $x_1$ to $58$ on $x_2$ and when inspecting the auto correlation plots given by Figures [\ref{fig:auto_mala_d2_pi2_1},\ref{fig:auto_mala_d2_pi2_2},\ref{fig:auto_mala_d64_pi2_1},\ref{fig:auto_mala_d64_pi2_2}] highlighting the suboptimal performance of MALA when targeting $\Pi_2$.
\begin{figure}[!htp]
\centering
\begin{minipage}{.5\textwidth}
  \centering
  \includegraphics[width=\linewidth]{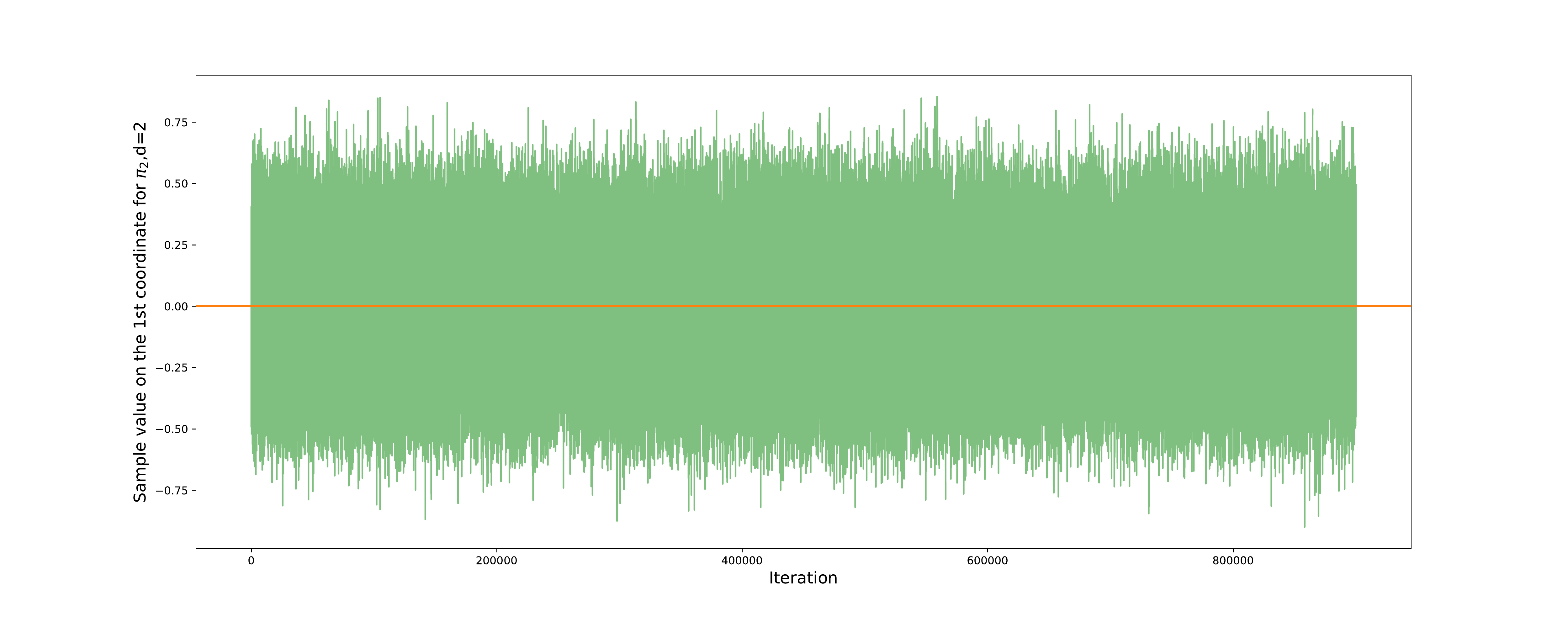}
  \caption{Traceplot on $x_1$ MALA chain targeting $\Pi_2$,$d=2$}
    \label{fig:mala_d2_pi2_1}
\end{minipage}%
\begin{minipage}{.5\textwidth}
  \centering
  \includegraphics[width=\linewidth]{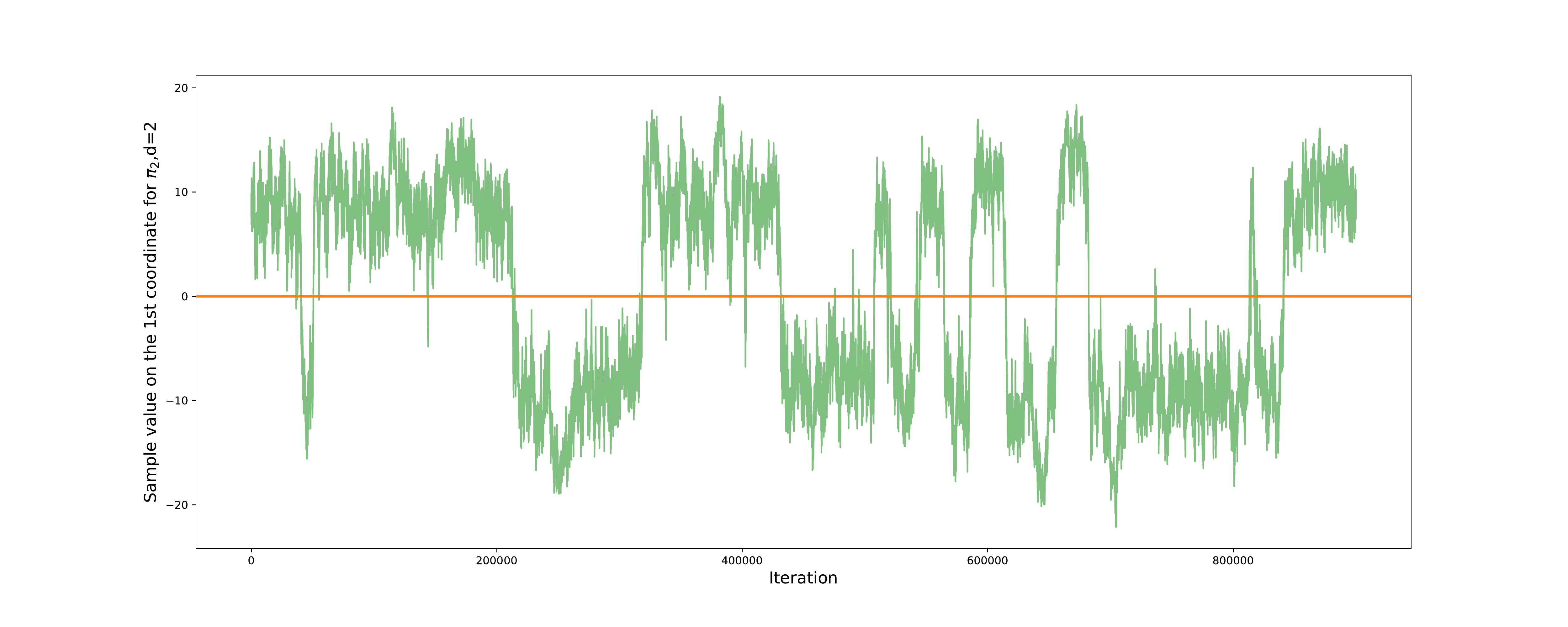}
   \caption{Traceplot on $x_2$ MALA chain targeting $\Pi_2$}
    \label{fig:mala_d2_pi2_2}
\end{minipage}
\end{figure}
\\
\\

We now turn our attention to the performance of our proposed algorithm, that is MAO algorithm, we consider the task of sampling from $\Pi_1$, our first impression is that the performance of MAO is comparable to that of MALA when the dimension is set at $d=2$, this is further confirmed when inspecting at the (ESS) from Table[\ref{tab:mao_pi1}], we also notice that the performance of MAO scales better when the dimension increases as highlighted by the (ESS) of $1019$ compared to that of $67$ for MALA, this is further highlighted by the traceplot given in Figure[\ref{fig:mao_d64}] that shows better mixing quality along the first coordinate. 

The main discrepancy between the two algorithms is when considering the task of sampling from $\Pi_2$.Indeed, we notice that MAO is able to bypass the failure of MALA when the target distribution decays faster than an Gaussian distribution. This better highlighted when considering tthe traceplots given Figures [\ref{fig:mao_d64_pi2_1},\ref{fig:mao_d64_pi2_2}] and the (ESS) given by Table [\ref{tab:mao_pi2}]. We also notice that the scaling of the performance of MAO with the dimension is better behaved when compared with MALA.

In order to obtain convergence diagnostics, we have used \textbf{Coda} package in \textbf{R} to run summary statistics on the different chains considered in order to investigate the convergence of the chains to the respective target distributions. Tables [\ref{tab:mala_pi1},\ref{tab:mala_pi2},\ref{tab:mao_pi1},\ref{tab:mao_pi2}] highlights the results obtained. 

\begin{figure}[htp]
\centering
\begin{minipage}{.5\textwidth}
  \centering
  \includegraphics[width=\linewidth]{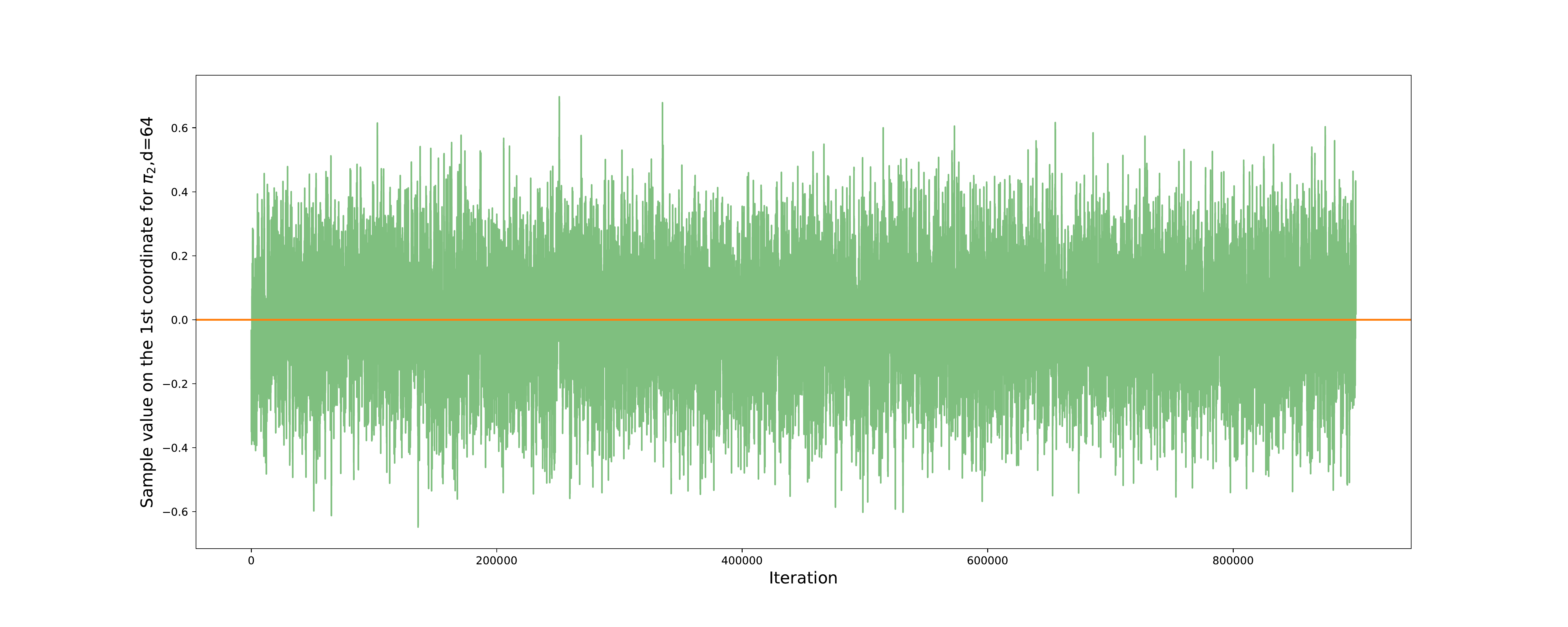}
  \caption{Traceplot on $x_1$ MALA chain targeting $\Pi_2$,$d=64$}
    \label{fig:mala_d64_pi2_1}
\end{minipage}%
\begin{minipage}{.5\textwidth}
  \centering
  \includegraphics[width=\linewidth]{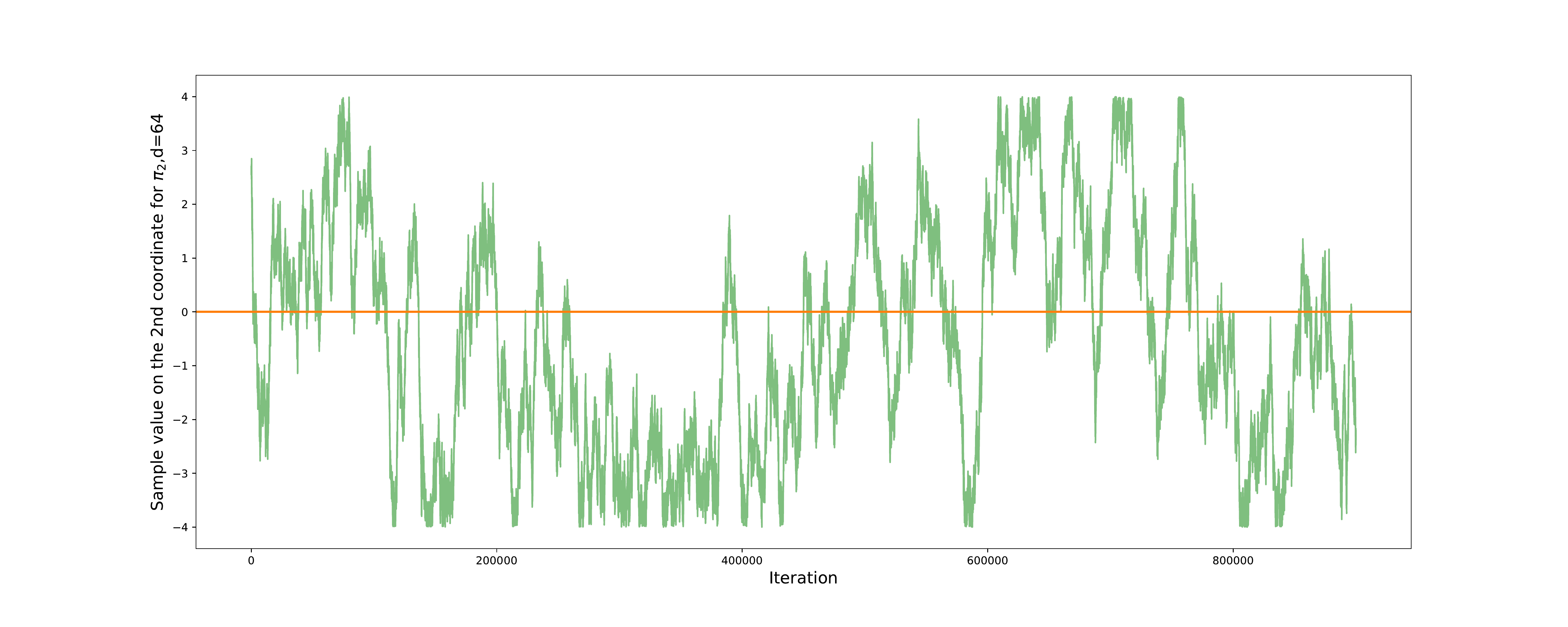}
   \caption{Traceplot on $x_2$ MALA chain targeting $\Pi_2$}
    \label{fig:mala_d64_pi2_2}
\end{minipage}
\end{figure}
\begin{figure}[htp]
\centering
\begin{minipage}{.5\textwidth}
  \centering
  \includegraphics[width=.7\linewidth]{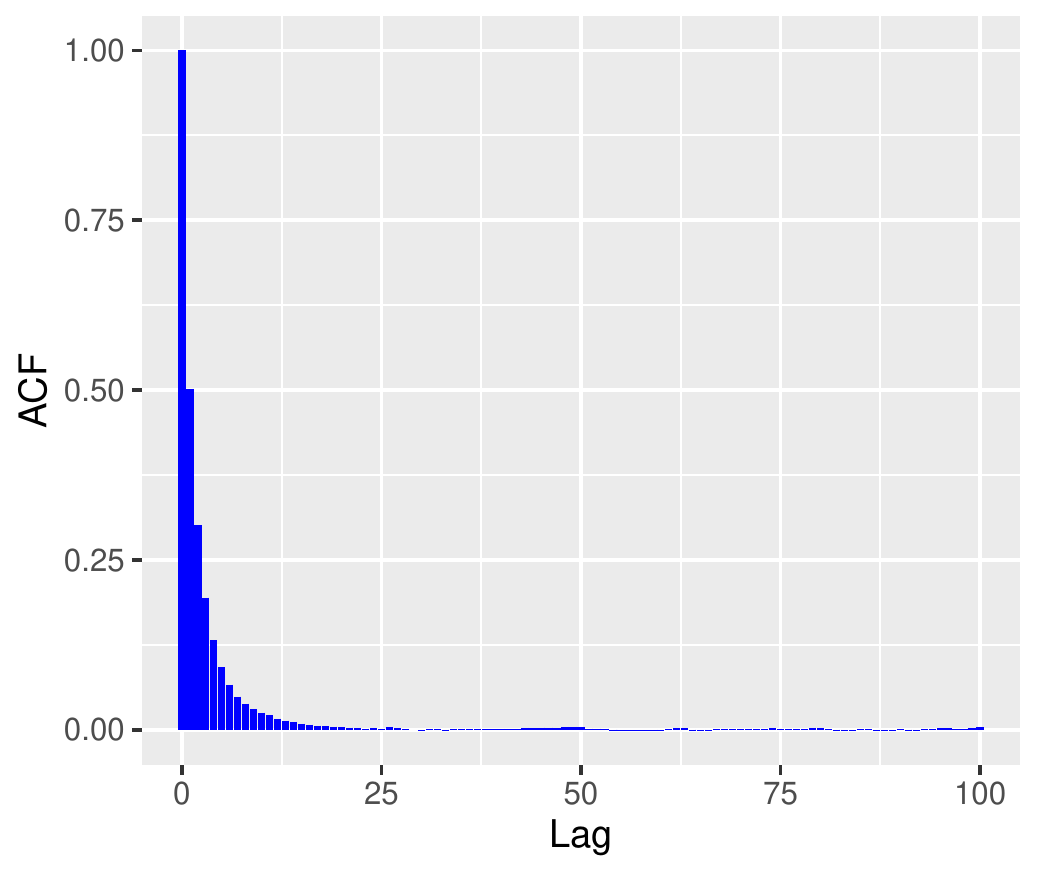}
  \caption{auto-correlation plot MALA $\Pi_2$ on $x_1$, $d=2$}
    \label{fig:auto_mala_d2_pi2_1}
\end{minipage}%
\begin{minipage}{.5\textwidth}
  \centering
  \includegraphics[width=.7\linewidth]{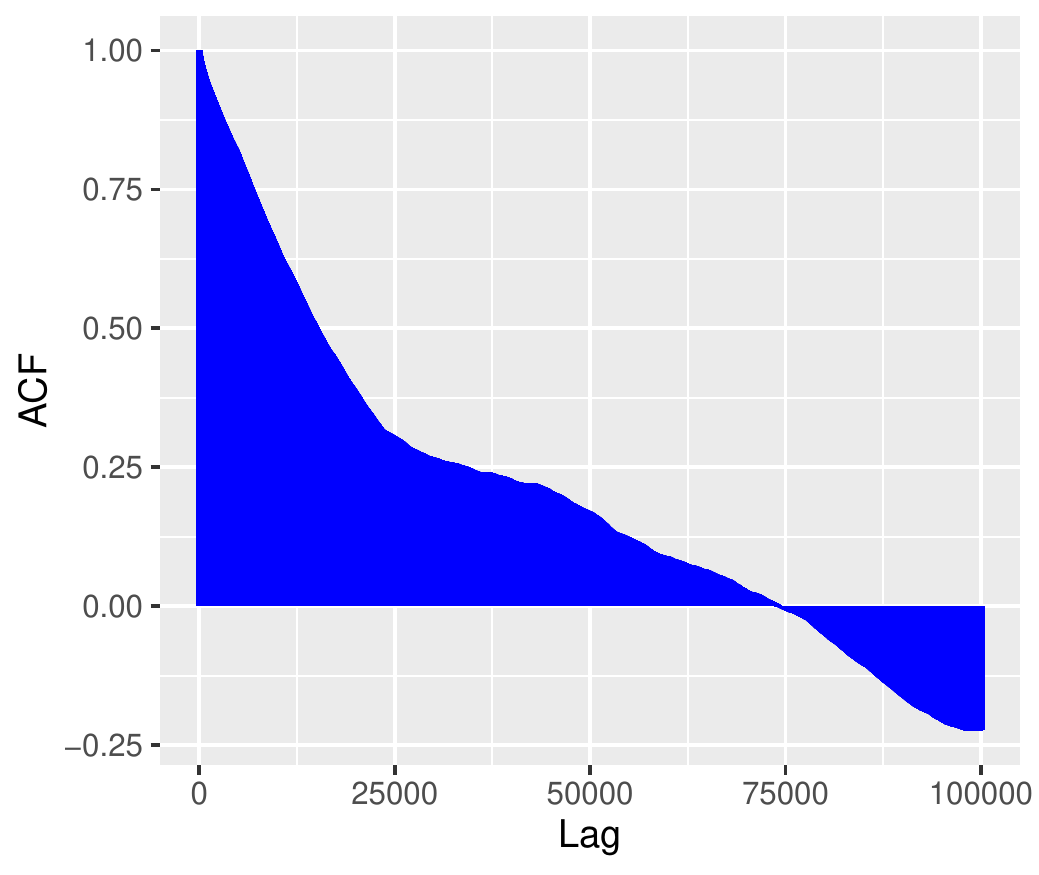}
   \caption{auto-correlation plot MALA  $\Pi_2$ on $x_2$}
    \label{fig:auto_mala_d2_pi2_2}
\end{minipage}
\end{figure}
\begin{figure}[htp]
\centering
\begin{minipage}{.5\textwidth}
  \centering
  \includegraphics[width=.7\linewidth]{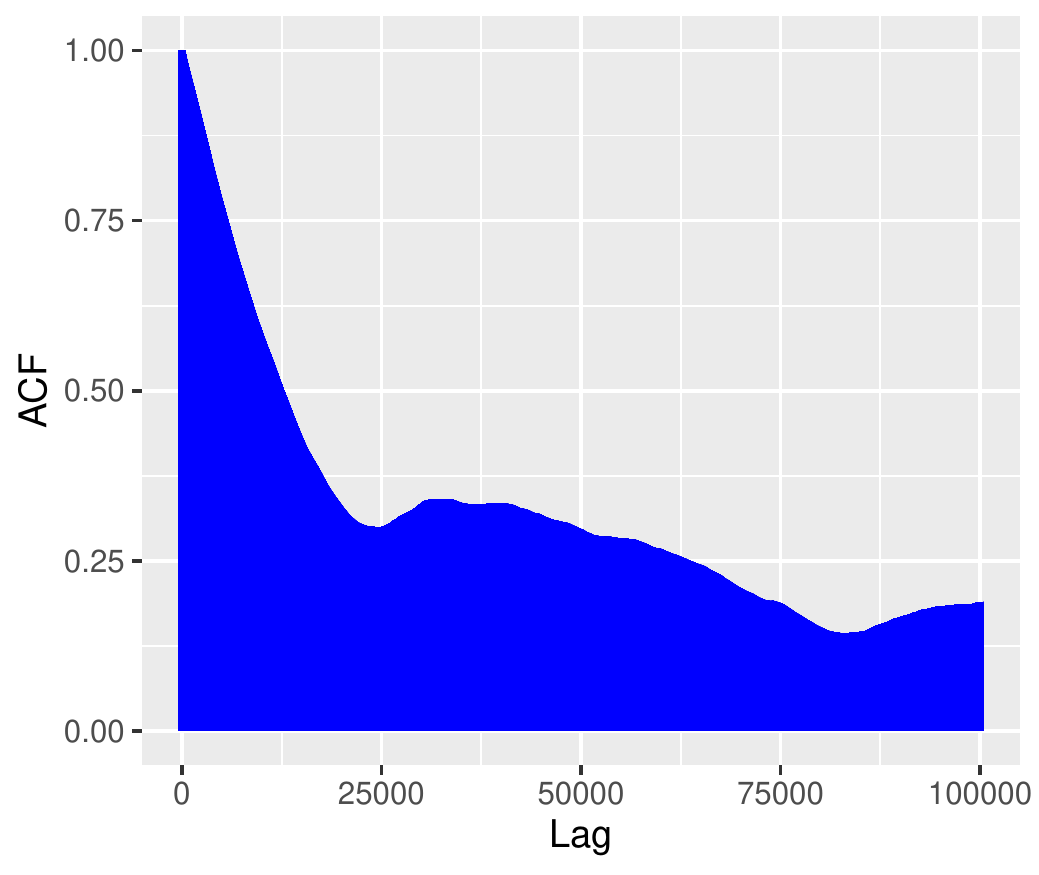}
  \caption{auto-correlation plot MALA $\Pi_2$ on $x_1$,$d=64$}
    \label{fig:auto_mala_d64_pi2_1}
\end{minipage}%
\begin{minipage}{.5\textwidth}
  \centering
  \includegraphics[width=.7\linewidth]{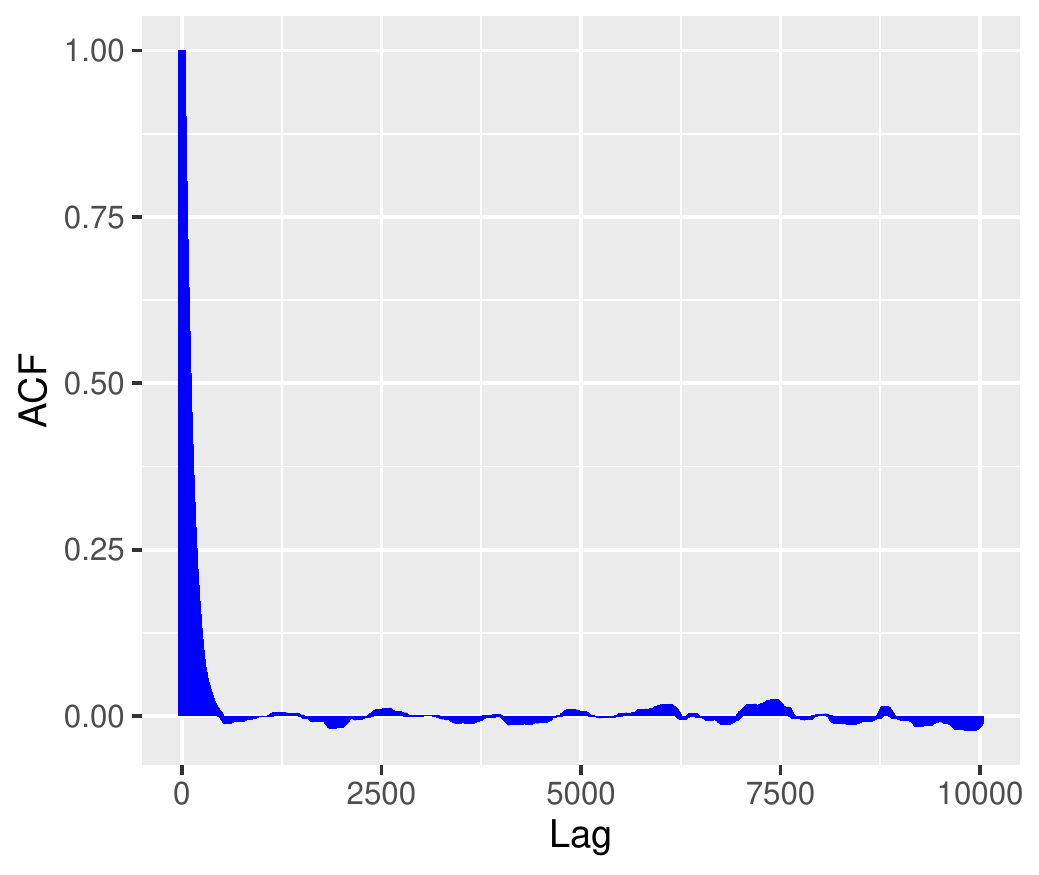}
   \caption{auto-correlation plot MALA  $\Pi_2$ on $x_2$}
    \label{fig:auto_mala_d64_pi2_2}
\end{minipage}
\end{figure}

\begin{figure}[htp]
\centering
\begin{minipage}{.5\textwidth}
  \centering
  \includegraphics[width=\linewidth]{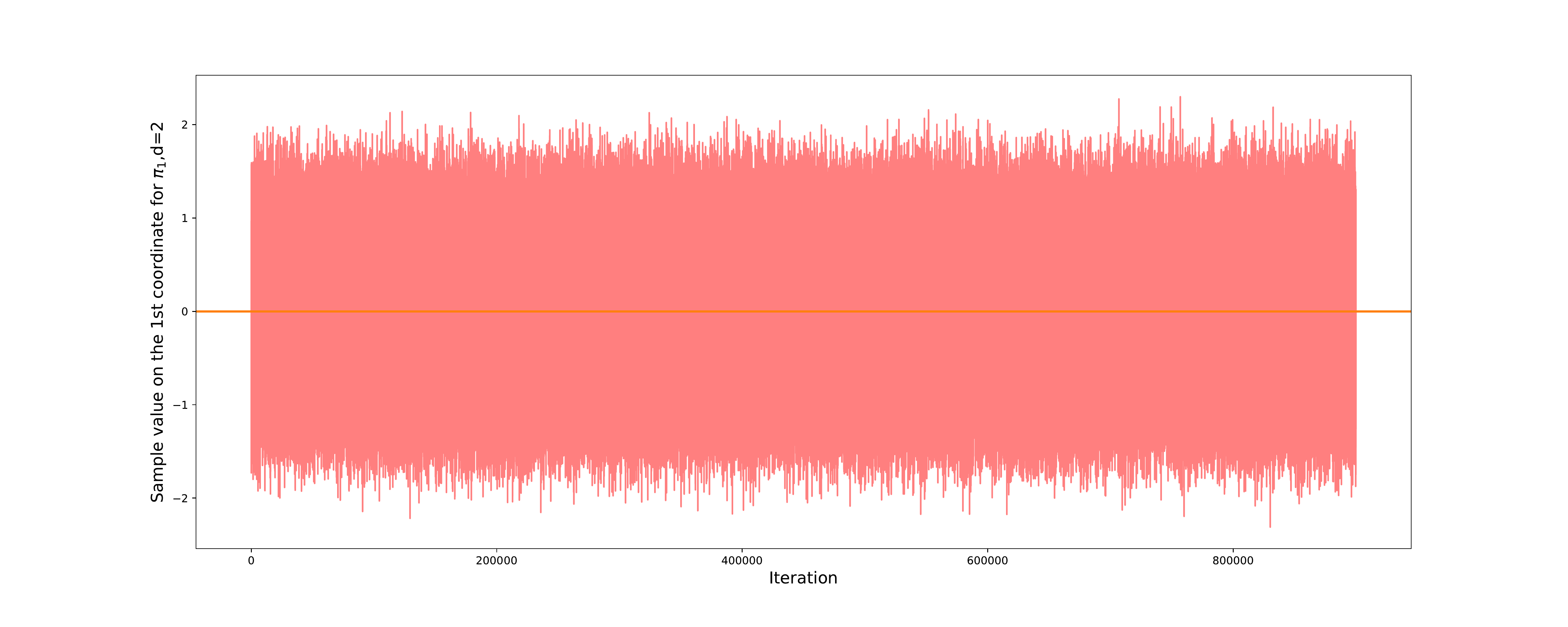}
  \caption{Traceplot on $x_1$ MAO chain targeting $\Pi_1$,$d=2$}
    \label{fig:trace_mao_d2_pi1}
\end{minipage}%
\begin{minipage}{.5\textwidth}
  \centering
  \includegraphics[width=.7\linewidth]{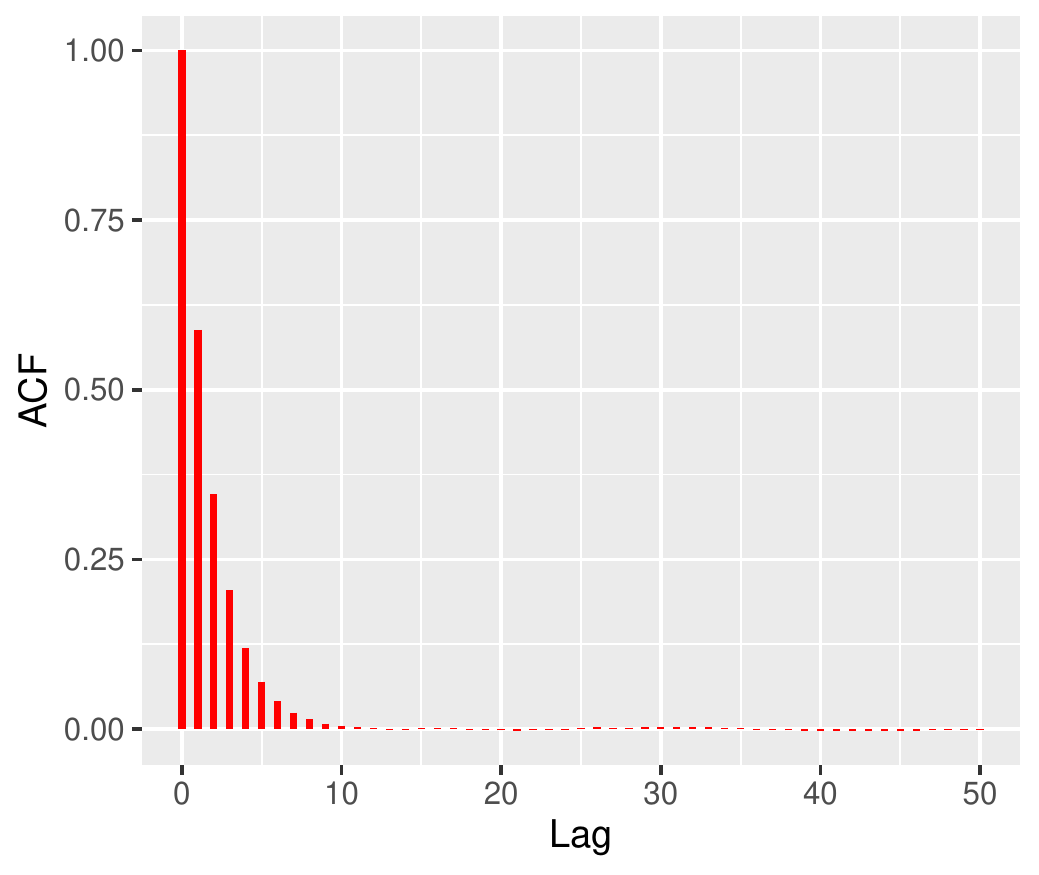}
   \caption{auto-correlation plot MAO chain targeting $\Pi_1$}
    \label{fig:auto_mao_d2_pi1}
\end{minipage}
\end{figure}
\begin{figure}[htp]
\centering
\begin{minipage}{.5\textwidth}
  \centering
  \includegraphics[width=\linewidth]{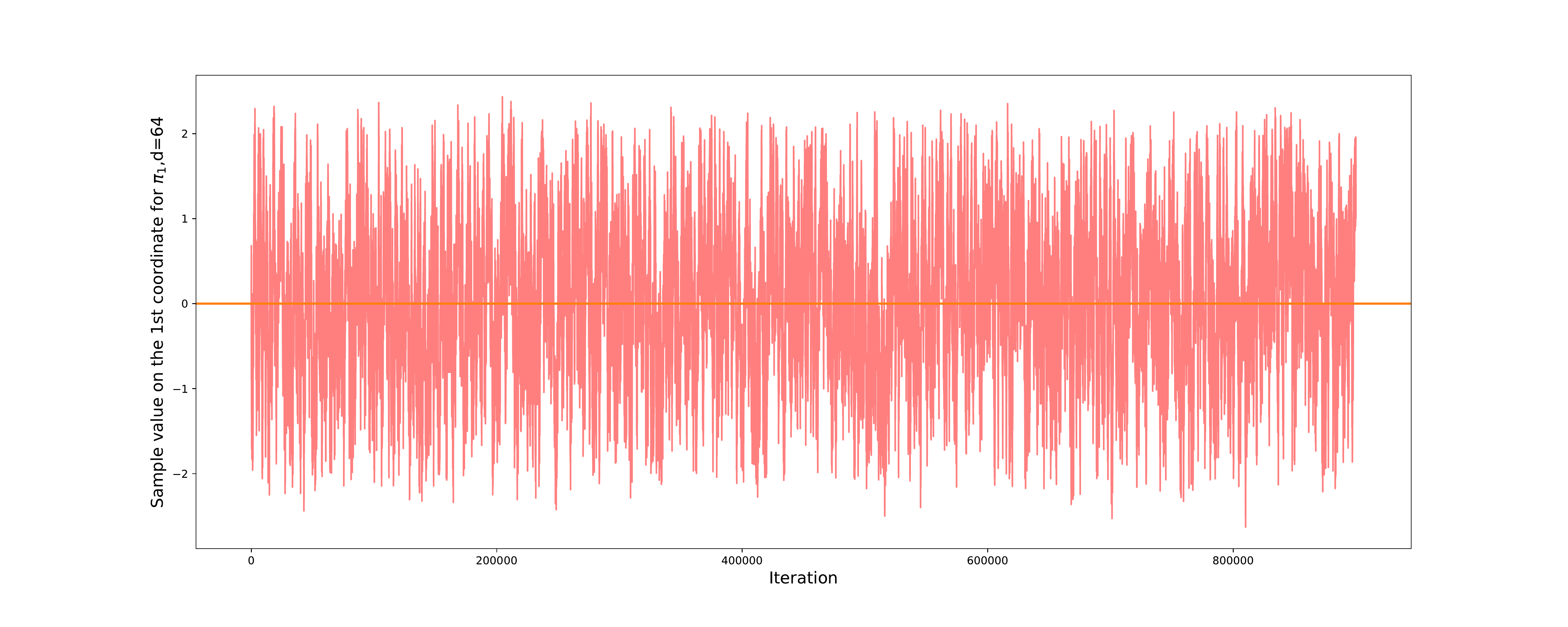}
  \caption{\small{Traceplot on $x_1$ MAO chain targeting $\Pi_1$,$d=64$}}
    \label{fig:mao_d64}
\end{minipage}%
\begin{minipage}{.5\textwidth}
  \centering
  \includegraphics[width=.7\linewidth]{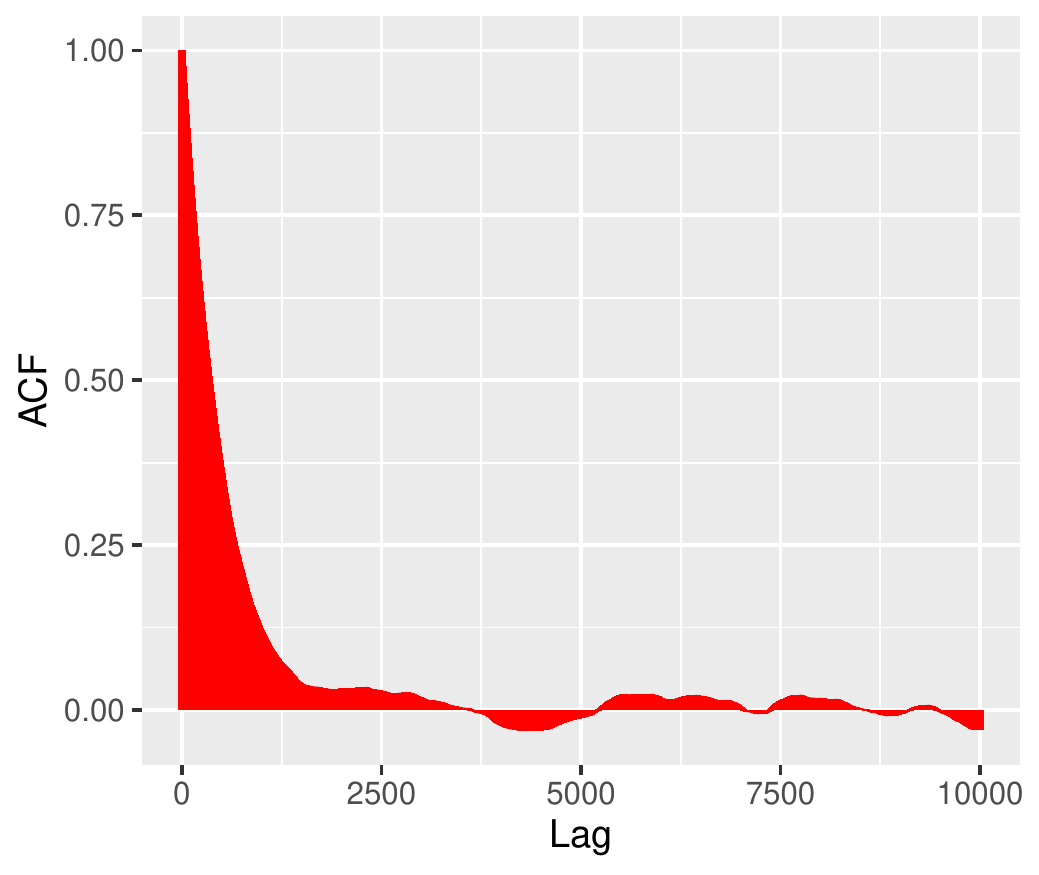}
   \caption{\small{Auto-correlation plot MAP chain targeting $\Pi_1$}}
    \label{fig:auto_mao_d64}
\end{minipage}
\end{figure}
\begin{figure}[htp]
\centering
\begin{minipage}{.5\textwidth}
  \centering
  \includegraphics[width=\linewidth]{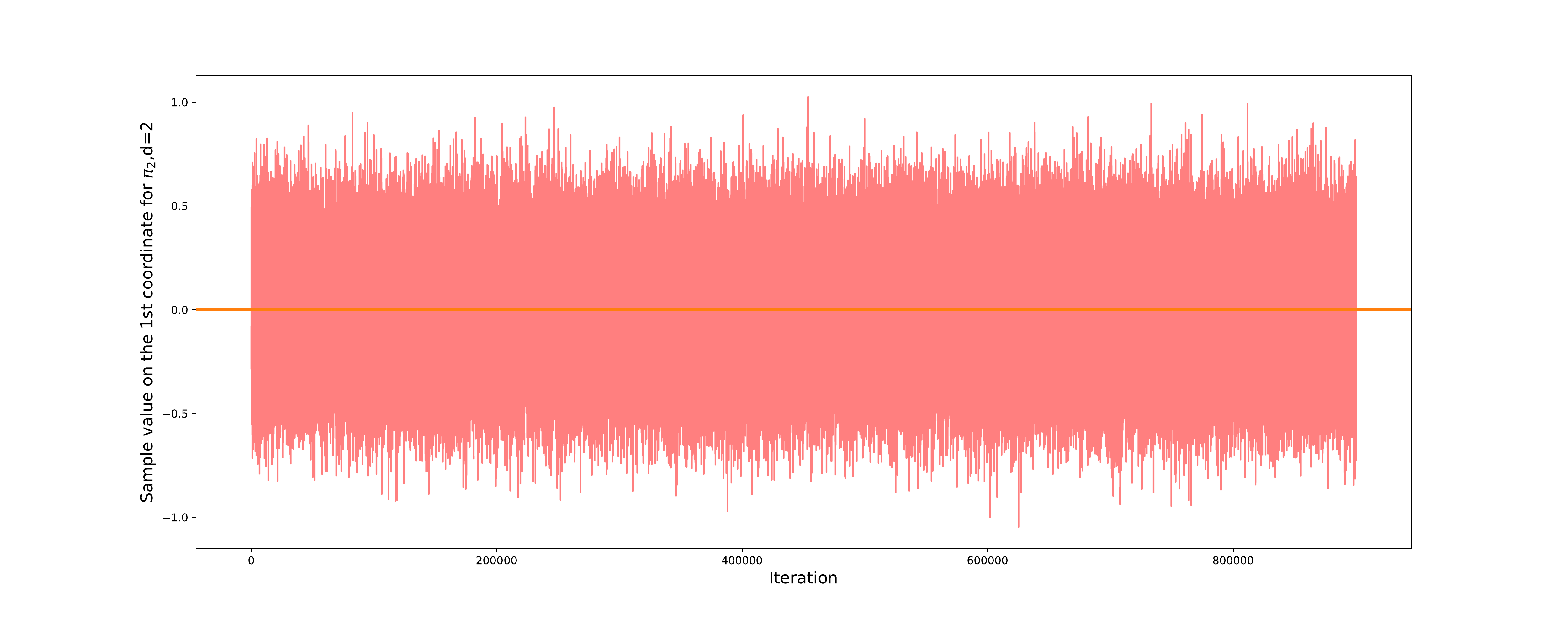}
  \caption{\small{Traceplot on $x_1$ MAO chain targeting $\Pi_2$,$d=2$}}
    \label{fig:mao_d2_pi2_1}
\end{minipage}%
\begin{minipage}{.5\textwidth}
  \centering
  \includegraphics[width=\linewidth]{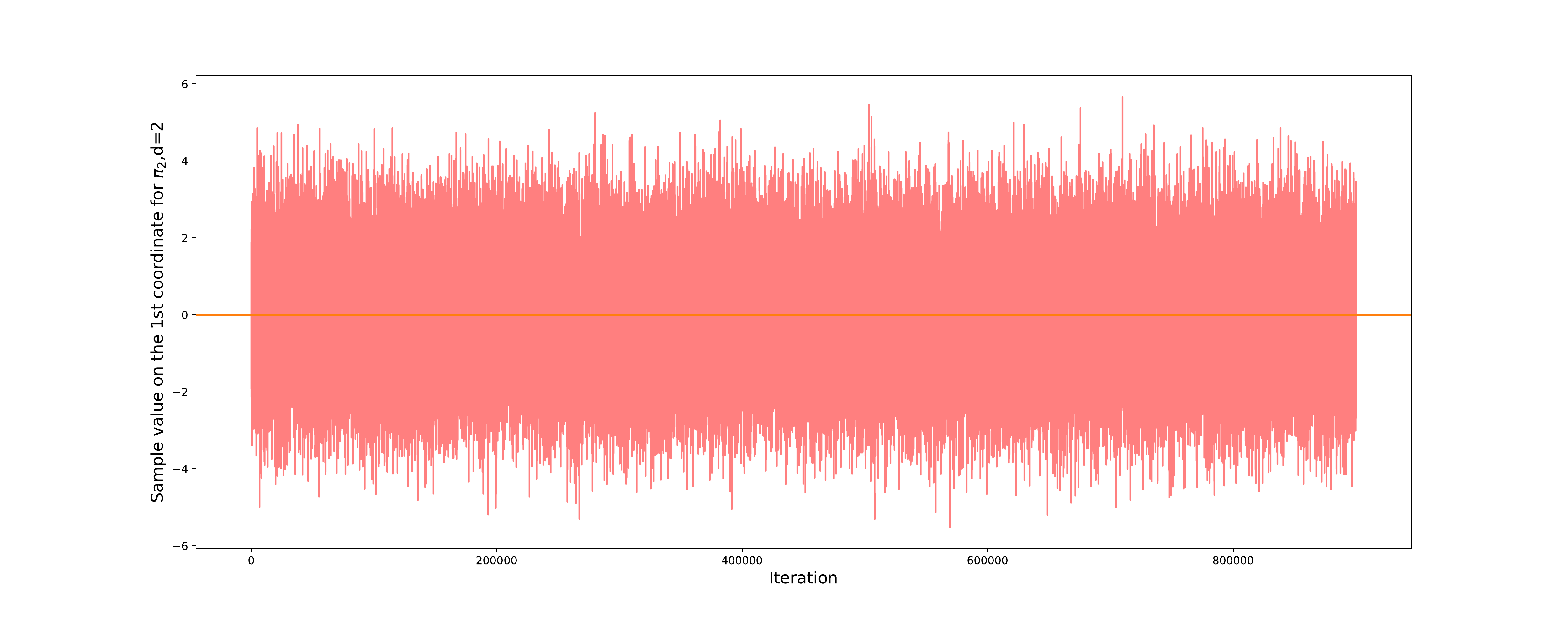}
   \caption{\small{Traceplot on $x_2$ MAO chain targeting $\Pi_2$,$d=2$}}
    \label{fig:mao_d2_pi2_2}
\end{minipage}
\end{figure}
\begin{figure}[htp]
\centering
\begin{minipage}{.5\textwidth}
  \centering
  \includegraphics[width=\linewidth]{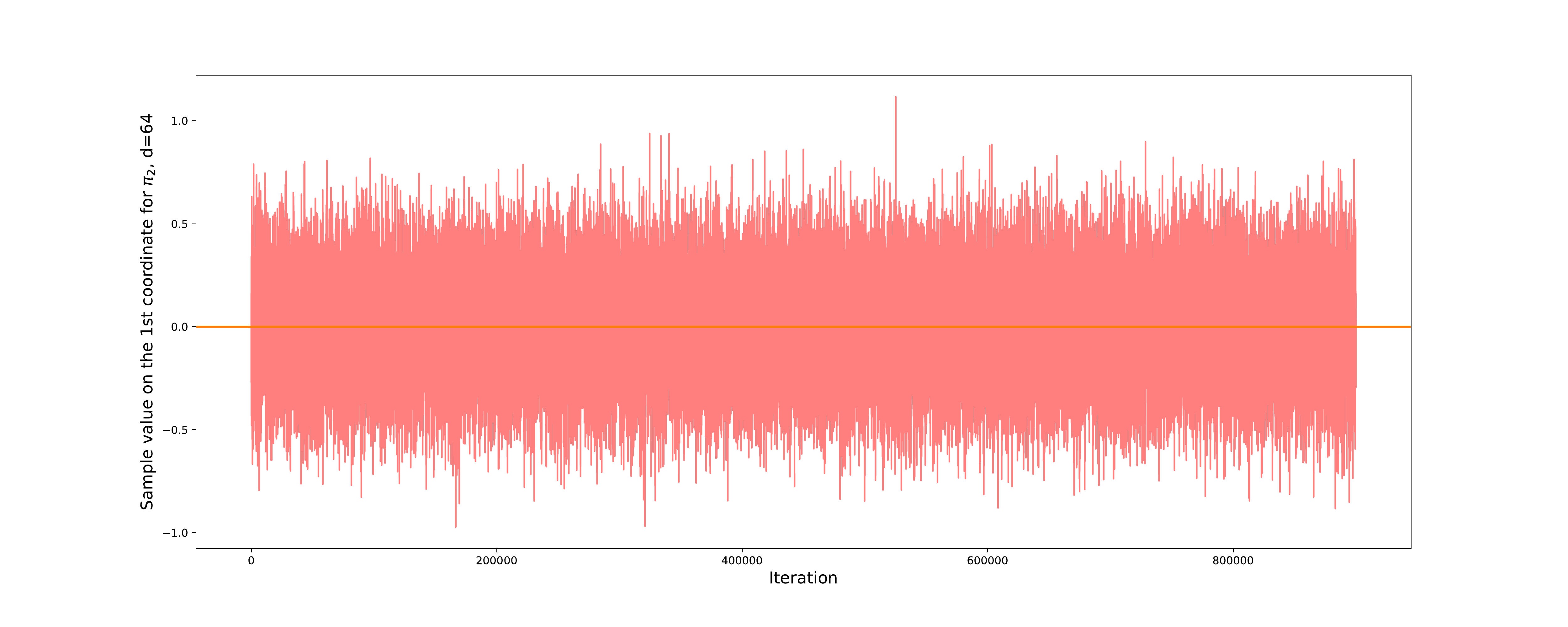}
  \caption{\small{Traceplot on $x_1$ MAO chain targeting $\Pi_2$,$d=64$}}
    \label{fig:mao_d64_pi2_1}
\end{minipage}%
\begin{minipage}{.5\textwidth}
  \centering
  \includegraphics[width=\linewidth]{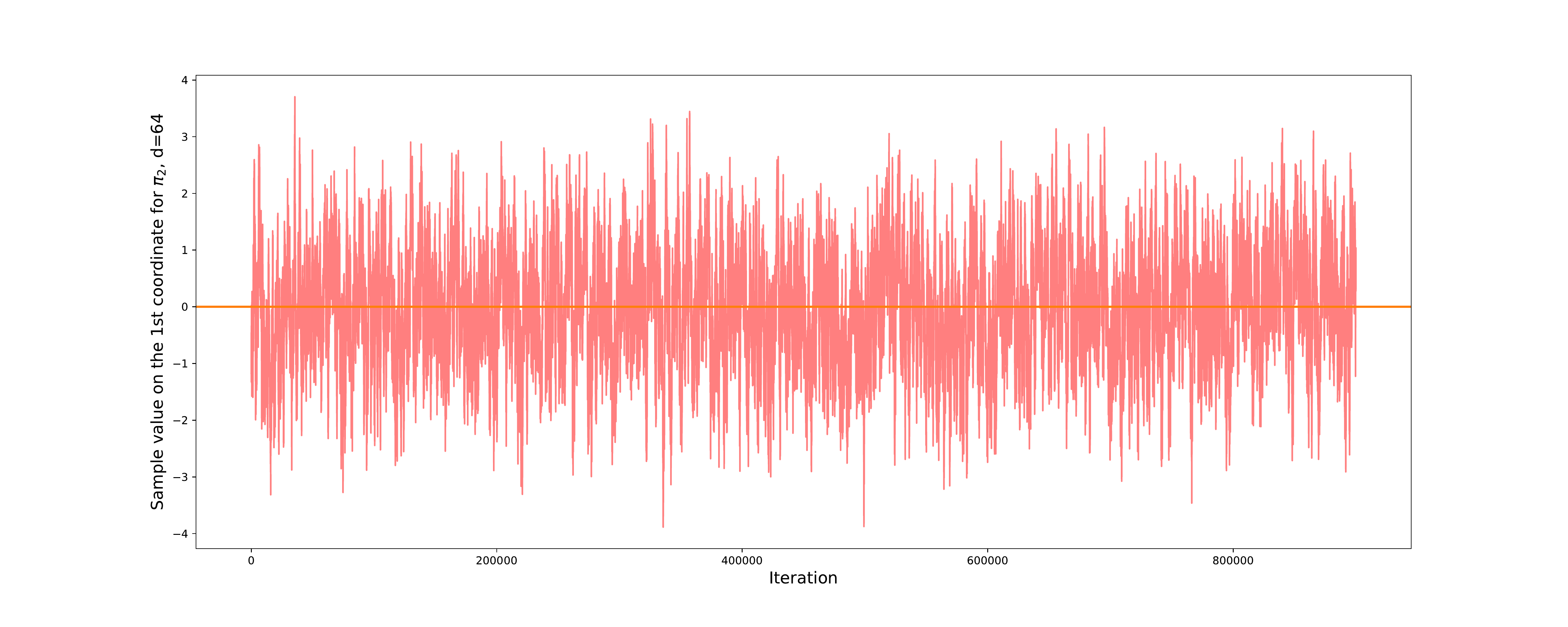}
   \caption{\small{Traceplot on $x_2$ MAO chain targeting $\Pi_2$,$d=64$}}
    \label{fig:mao_d64_pi2_2}
\end{minipage}
\end{figure}

\begin{figure}[htp]
\centering
\begin{minipage}{.5\textwidth}
  \centering
  \includegraphics[width=.7\linewidth]{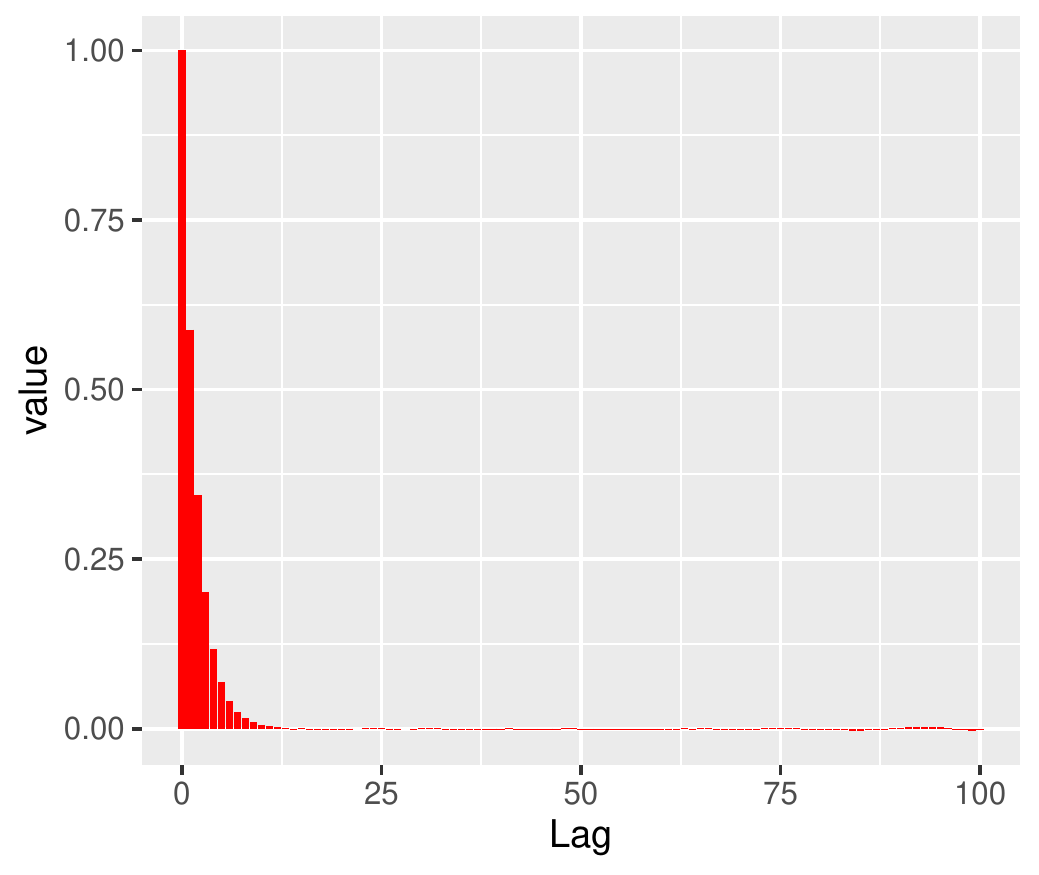}
  \caption{auto-correlation plot MAO $\Pi_2$ on $x_1$, $d=2$}
    \label{fig:auto_mao_d2_pi2_1}
\end{minipage}%
\begin{minipage}{.5\textwidth}
  \centering
  \includegraphics[width=.7\linewidth]{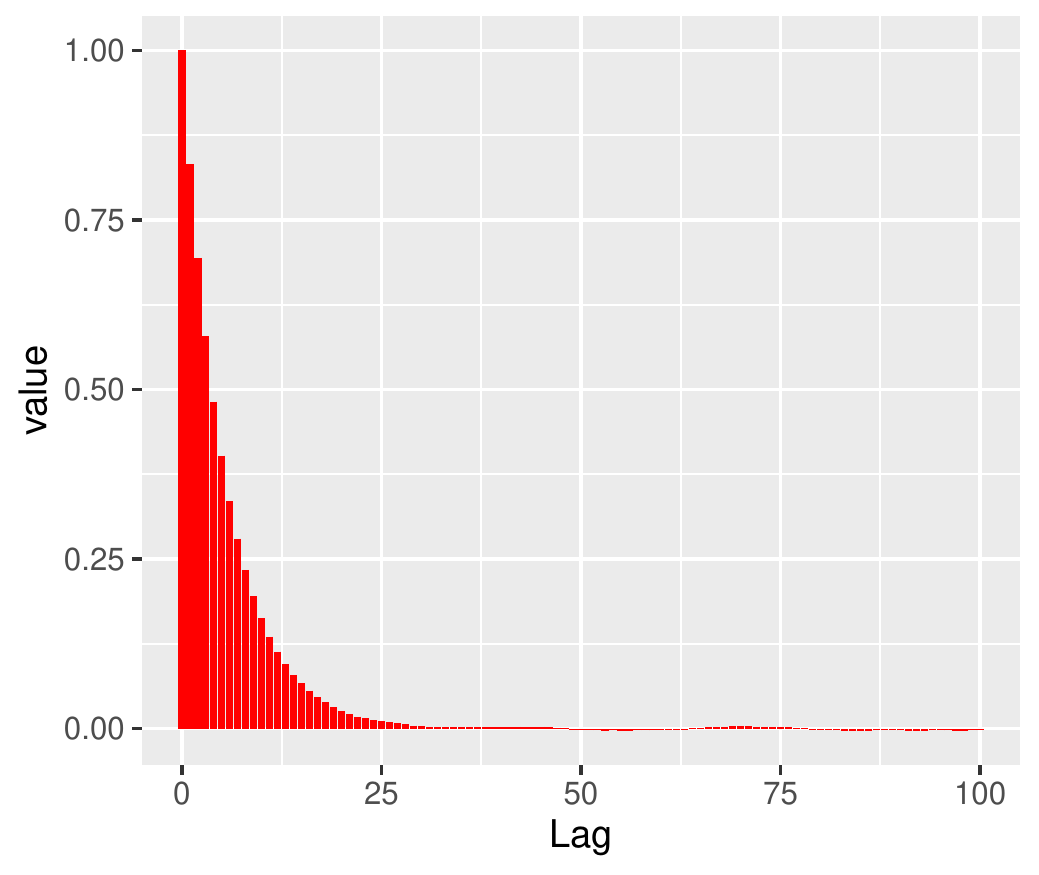}
   \caption{auto-correlation plot MAO  $\Pi_2$ on $x_2$, $d=2$}
    \label{fig:auto_mao_d2_pi2_2}
\end{minipage}
\end{figure}


\begin{figure}[htp]
\centering
\begin{minipage}{.5\textwidth}
  \centering
  \includegraphics[width=.7\linewidth]{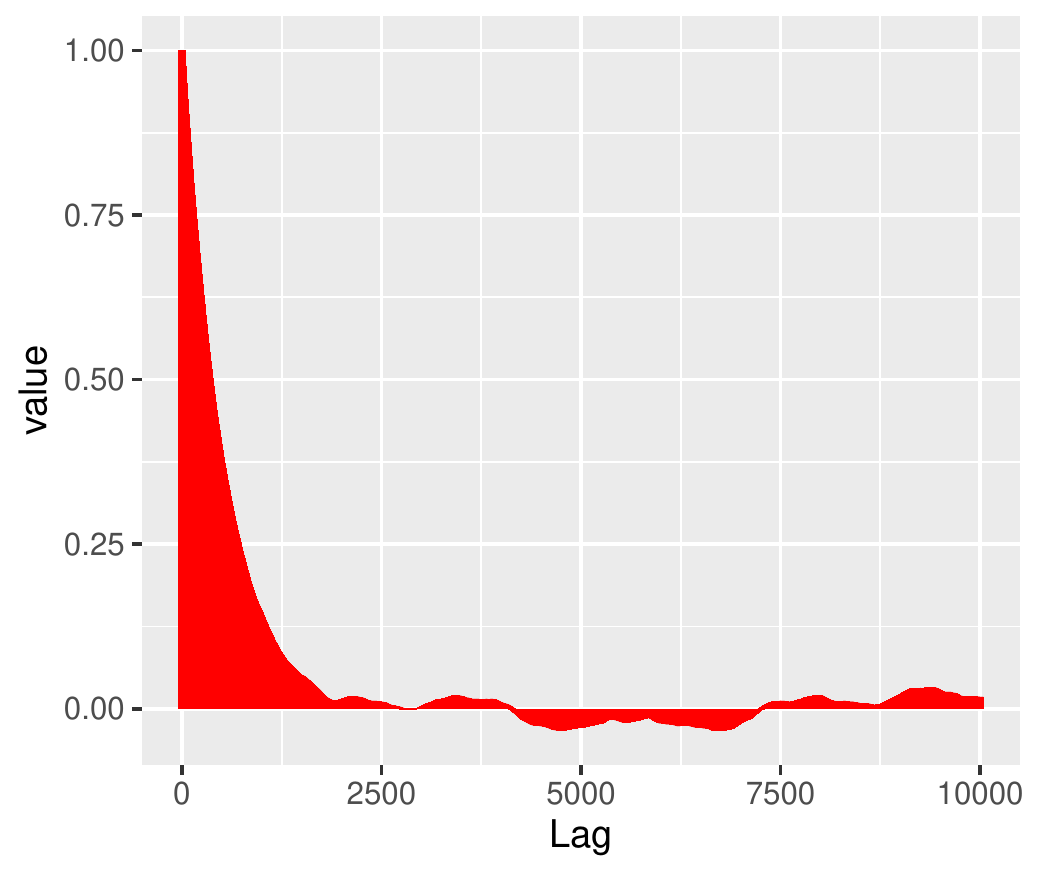}
  \caption{auto-correlation plot MAO $\Pi_2$ on $x_1$, $d=64$}
    \label{fig:auto_mao_d64_pi2_1}
\end{minipage}%
\begin{minipage}{.5\textwidth}
  \centering
  \includegraphics[width=.7\linewidth]{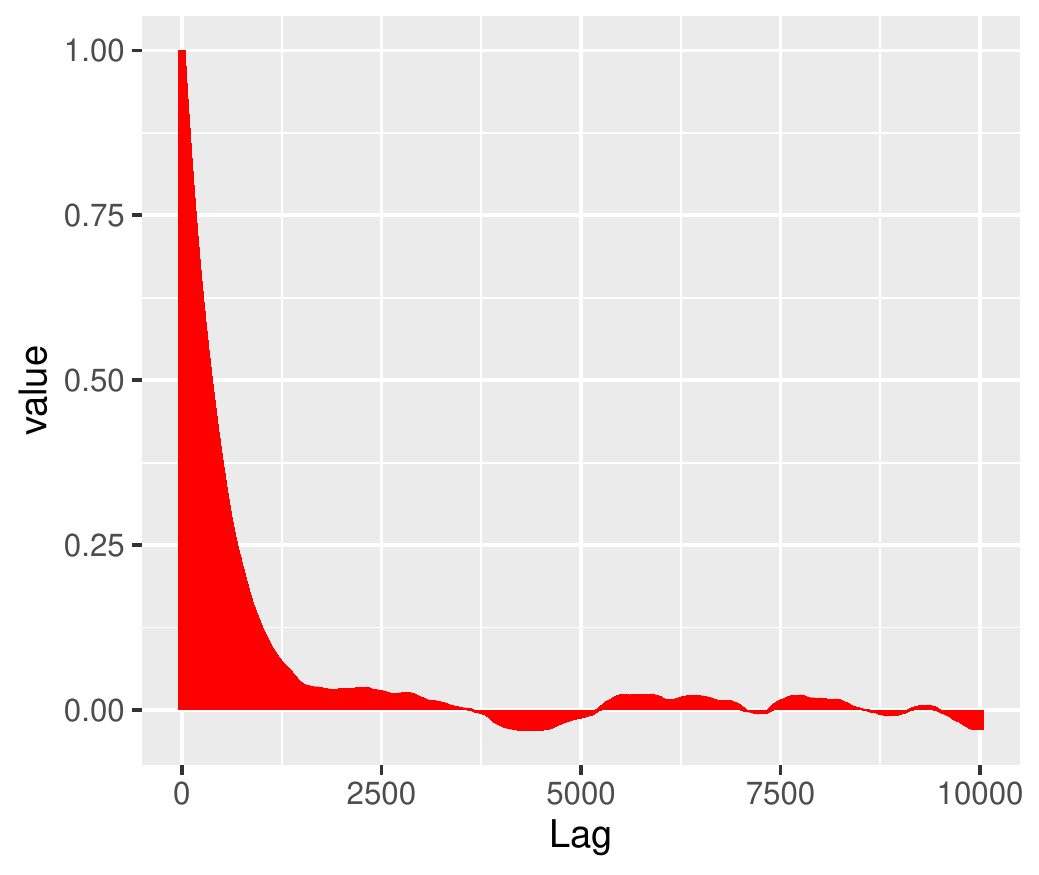}
   \caption{auto-correlation plot MAO  $\Pi_2$ on $x_2$, $d=64$}
    \label{fig:auto_mao_d64_pi2_2}
\end{minipage}
\end{figure}

\begin{table}[!ht]
     \centering
     \begin{tabular}{llllllll}
    \toprule
   dimension & $d=2$ & $d=4$  & $d=8$ & $d=16$ & $d=32$ & $d=64$ & $d=128$   \\
    \midrule
    ESS on $x_1$ & $422774$ & $59919$ & $10518$ & $1782$ & $335$ & $67$ & $20$\\
    \midrule
    ESS on $x_2$ & $426947$ & $59180$ & $10343$ & $1815$ & $346$ & $69$ & $24$ \\
    \midrule
    Acceptance rate & $0.827$ & $0.863$ & $0.864$ & $0.880$ & $0.899$ & $0.913$ &  $0.834$ \\
     \bottomrule
    \end{tabular}
    \caption{Convergence Diagnostic for MALA $\pi_1 $ chain  with $1 000 000$ samples}
     \label{tab:mala_pi1}
\end{table}

\begin{table}[!ht]
     \centering
     \begin{tabular}{llllllll}
    \toprule
   dimension & $d=2$ & $d=4$  & $d=8$ & $d=16$ & $d=32$ & $d=64$ & $d=128$   \\
    \midrule
    ESS on $x_1$ & $422774$ & $246566$ & $141180$ & $35693$ & $10042$ & $4293$ & $980$ \\
    \midrule
    ESS on $x_2$ & $426947$ & $58$ & $1006$ & $374$ & $487$ & $25$ & $24$ \\
    \midrule
    Acceptance rate & $0.827$ & $0.317$ & $0.495$ & $0.579$ & $0.581$ & $0.902$ &  $0.581$ \\
     \bottomrule
    \end{tabular}
    \caption{Convergence Diagnostic for MALA $\pi_2$ chain with $1 000 000$ samples }
     \label{tab:mala_pi2}
\end{table}

\begin{table}[!ht]
     \centering
     \begin{tabular}{llllllll}
    \toprule
   dimension & $d=2$ & $d=4$  & $d=8$ & $d=16$ & $d=32$ & $d=64$ & $d=128$   \\
    \midrule
    ESS on $x_1$ & $260593$ & $90468$ & $26462$ & $8306$ & $2781$ & $1019$ & $98$0 \\
    \midrule
    ESS on $x_2$ & $259203$ & $90794$ & $26282$ & $8306$ & $2831$ & $1016$ & $24$ \\
    \midrule
    Acceptance rate & $0.632$ & $0.695$ & $0.716$ & $0.740$ & $0.774$ & $0.816$ &  $0.799$ \\
     \bottomrule
    \end{tabular}
    \caption{Convergence Diagnostic for MAO $\pi_1$ chain with $1 000 000$ samples and step size given by Theorem~\ref{thm_b} }
     \label{tab:mao_pi1}
\end{table}

\begin{table}[!ht]
     \centering
     \begin{tabular}{llllllll}
    \toprule
   dimension & $d=2$ & $d=4$  & $d=8$ & $d=16$ & $d=32$ & $d=64$ & $d=128$   \\
    \midrule
    ESS on $x_1$ & $260593$ & $90468$ & $26462$ & $8306$ & $2781$ & $1019$ & $391$ \\
    \midrule
    ESS on $x_2$ & $259203$ & $90794$ & $26282$ & $8306$ & $2831$ & $1016$ & $351$ \\
    \midrule
    Acceptance rate & $0.302$ & $0.448$ & $0.603$ & $0.732$ & $0.824$ & $0.896$ &  $0.02$ \\
     \bottomrule
    \end{tabular}
    \caption{Convergence Diagnostic for MAO $\pi_2$ chain with $1 000 000$ samples and step size given by Theorem ~\ref{thm_b} }
     \label{tab:mao_pi2}
\end{table}

\newpage

\section{Proofs}
\label{sec:proof_of_the_main_theorem}

In this section, we detail the  proof of
Theorem~\ref{thm_a}, Theorem~\ref{thm_b} and Theorem~\ref{thm:opti_mao}.  In order to do so, we will follow a similar approach as the proof structure of ~\cite{dwivedi2018log} and ~\cite{chenhmc} with some minor modifications. We use the isoperimetric inequalities, that we describe in the next paragraph. We will leverage the framework on mixing times based on the conductance profile, mainly we invoke Lemma 4 from~\cite{chenhmc}.  We then seek to apply Lemma 2~\cite{chenhmc} to establish a bound on the conductance profile. In a similar fashion as the proofs in previous works, we need to control the overlap between the proposal
distributions of MAO at two nearby points and show that the
Metropolis-Hastings correction step only serves to modify the proposal distribution by a controllable quantity. This control is provided by
Lemma~\ref{lem:transition_closeness}.  We use it to prove
Theorem~\ref{thm_a} and Theorem~\ref{thm_b}. Finally, we provide a proof of Theorem~\ref{thm:opti_mao}, all while providing a proof of a special sub-family of $\mathcal{E}(\alpha,m)$ that is $(\alpha,\gamma,m)$-\textit{strongly log-concave}. 
Similarly to the proofs  
\paragraph{Isoperimetric Inequalities.}

We describe here the isoperimetric inequalities as described in the~\cite{chenhmc}:

A distribution $\Pi$ supported on  $\mathcal{X} \subset \mathbb{R}^d$ satisfies the \emph{isoperimetric inequality} ($\mathfrak{a} =
0$) or the \emph{log-isoperimetric inequality} ($\mathfrak{a}=
\frac{1}{2}$) with constant $\psi_{a}$ if for given any partition $S_1$, $S_2$, $S_3$ of $\mathcal{X}$ we have:
  \begin{align}
    \label{eq:assumption_isoperimetric}
    \Pi(S_3)\geq\frac{1}{2\psi_{a}}\cdot d(S_1,S_2)\cdot \min(\Pi(S_1),\Pi(S_2))\cdot \log^{a}\left(1+\frac{1}{\min(\Pi(S_1),\Pi(S_2))}\right)
  \end{align}
 
$d(S_1,S_2)=\inf_{x\in S_1 y\in S_2}\{\left\|x-y\right\|_{2}\}$ being the distance between two sets $S_1, S_2$.

 $\Pi_\Omega$  is the restriction to $\Omega$ of  $\Pi$, whose  density $\pi_\Omega(x) =
  \frac{\pi(x)\mathbf{1}_\Omega (x)}{\Pi(\Omega)}$.
  
Lemma 16 of \cite{chenhmc} and the discussion afterwards has shown the following key result, that we will use repeatedly.

\begin{lem}\label{lem:isoperimetry}
Suppose that $\Omega$ is a convex set, and the potential of $\Pi$ is $m$-strongly convex. Then the restriction $\Pi_\Omega$ satisfies the log-isoperimetric inequality with $\mathfrak{a}=\frac{1}{2}$ and constant $\psi_{a}=\frac{1}{\sqrt{m}}$.
\end{lem}
This lemma shows that log-isoperimetry follows from strong convexity, hence it holds under our Assumptions \ref{itm:assumptionA} and \ref{itm:assumptionB}.

Before controlling the acceptance probability and bounding the conductance profile, we first provide an example of a convex set with high probability mass, for target distributions $\Pi$ with an $m$-strongly convex potential, this essentially corresponds to Lemma 5 of \cite{dwivedi2018log} (although that result was only stated for $s\in (0,1/2)$, so we've slightly modified the definition of $r$ to make it applicable to every $s\in (0,1)$.
\begin{lem}
\label{lem:growth}
Suppose that $\Pi$ is  $m$-strongly log-concave. For $s \in(0,1)$ let  $\mathcal{R}_{s,\frac{1}{2}}=\mathbb{B} \left(x^*, r(s)\sqrt{\frac{d}{m}}\right)$. Then we have that
$$    \Pi(\mathcal{R}_{s,\frac{1}{2}})\geq 1-s.$$
\end{lem}
This result has been used in previous works for establishing mixing times for MALA, and will be used when upper bounding the average acceptance probability for MAO, next we prove that that the overlap between the proposal distributions of MAO at two nearby points is controllable.
From our perspective, this result shows that Assumption \ref{itm:assumptionA} is a special case of Assumption \ref{itm:assumptionB}, with $\tau(s)=r(s)$, $\gamma=2$, and $\Omega=\mathcal{R}_{s,\frac{1}{2}}$. Hence it suffices to show Theorem \ref{thm_b}.

In the next two sections, we will state a key lemma about the overlaps between Markov kernels, and then prove Theorem \ref{thm_b}. The following two sections are dedicated to proving the two statements of the key lemma. The remaining sections prove Theorem \ref{thm:opti_mao} (applicable to situations where we apply MAO with n approximate optimum, instead of the true one), show some examples of potentials satisfying Assumption \ref{itm:assumptionB} with $\gamma>2$, and prove Proposition \ref{prop:feasiblestart} (which bounds the value of $\beta$ for some feasible starts).

\subsection{Overlap of MAO} 
This subsection is devoted to deriving bounds for the  MAO chain:
(1) first, by quantifying  the overlap between
 proposal distributions at two adjacent points, (2) secondly,  we show that the difference in the proposal distribution induced by the Metropolis-Hastings correction rule is controllable given an appropriate choice of the step size. Putting the two
pieces together enables us to invoke Lemma 2 of \cite{chenhmc} to prove
Theorem~\ref{thm_b}.

To that end, we first introduce some notations.  Let $\mathcal{T}$
be the transition operator of the MAO chain with step size $h$. Let $\mathcal{T}_x$ denote the proposal distribution at a given
$x \in \mathcal{X}$ prior to the accept-reject step and
the lazy step. Let $\mathcal{T}_{x}^{\textrm{Before-Lazy}}$ the same transition distribution after the proposal and the Metropolis-Hastings correction rule, before the lazy step. We have by definition 
\begin{align}
\label{eq:transition_lazy_decomposition}
  \mathcal{T}_x(A) = \zeta \delta_x(A) + (1-\zeta)
  \mathcal{T}_{x}^{\textrm{Before-Lazy}}(A) \qquad \mbox{for any measurable set $A \in
     \mathcal{B}(\mathcal{X})$.}
\end{align}
In our proofs, we invoke the set $\Omega$ defined in Assumption \ref{itm:assumptionB}. The main premise of our proof is the MAO chain has a good smooth behaviour inside the set $\Omega$ for a particular choice of $s$. The behaviour is less smooth outside such a set. However, since the target distribution has a low probability mass on said region, the chain rarely visits the area outside the set $\Omega$. Thus, it suffices to analyze the chain's behavior inside the ball in order to derive its mixing time. In the next lemma, we derive overlap bounds for the transition
distribution of the MAO chain.  Given some universal constant $c$, we
require
\begin{subequations}
\begin{align}
     \label{eq:step_condition_thm_proof1}
     h&\leq \frac{1}{c \tau (\frac{\epsilon^2}{3\beta})d^{\omega}} \quad \textrm{ for}\quad  \Pi \in \mathcal{E}(\alpha,\gamma,m), \quad \textrm{with} \quad \omega=\max
\left(\frac{2 (\alpha-1)}{\gamma},\frac{\gamma+\alpha-2}{\gamma}\right).
\end{align}
\end{subequations}

\begin{lem}
  \label{lem:transition_closeness}
  Consider a target distribution satisfying the conditions of Assumption~\ref{itm:assumptionB}. Then with the step size choice $h$ satisfying \eqref{eq:step_condition_thm_proof1} we have the following results: 
\begin{subequations}
\begin{align}
\label{eq:proposal_closeness}
 \sup_{\left\|x-y\right\|\leq \sqrt{\frac{h}{2}}} \left\|\mathcal{P}_{x}-\mathcal{P}_{y}\right\|_{TV}\leq \frac{1}{2}\\
  \label{eq:transition_proposal_overlap}
     \sup_{x\in \Omega} \left\|\mathcal{P}_{x}-\mathcal{T}_{x}^{\textrm{Before-Lazy}}\right\|_{TV}\leq \frac{1}{8}
\end{align}
\end{subequations}
\end{lem}
Lemma~\ref{lem:transition_closeness} is essential to the analysis of MAO
as it allows us to apply the conductance profile-based bounds of~\cite{chenhmc}.
It describes two important properties of  MAO. First,
from equation~\eqref{eq:proposal_closeness}, we see that the proposal distributions of MAO at a given two points are close if the given
two points are close. This can be proved by controlling the KL-divergence of the two proposal distributions of MAO. On the other hand, equation~\eqref{eq:transition_proposal_overlap} shows that the
Metropolis-Hastings correction  rule of MAO is well behaved inside $\Omega$ provided that  Assumption \ref{itm:assumptionB} holds.

We include the proof of the two claims of this Lemma in Sections \ref{sub:properties_of_the_hmc_proposal} and \ref{sub:proof_of_claim_eq:proposal_closeness}.

\subsection{Proof of Theorem \ref{thm_b}}
\label{sub:proof_thm2}
We now are equipped with the tools to prove Theorem \ref{thm_b}, we begin first by using the result of Lemma 2~\cite{chenhmc} and
Lemma~\ref{lem:transition_closeness} to derive an explicit bound for
on the MAO conductance profile. Given the assumptions of equation \eqref{eq:step_condition_thm_proof1} and the conditions of Assumption \ref{itm:assumptionB} hold, we are able to invoke the results of Lemma~\ref{lem:transition_closeness}.

Define the function $\Psi_{\Omega}: [0, 1] \mapsto \mathbb{R}_+$ as :
\begin{align}
\label{eq:psi_define}
  \Psi_{\Omega}(v) &= \begin{cases} \displaystyle\frac{1}{32}\cdot
    \min\left\{1, \frac{\sqrt{h}}{32\psi_{a}}
      \log^{a}\left(\frac{1}{v}\right)\right\} & \mbox{if $v \in
      \left[0, \frac{1-s}{2} \right]$.}  \\[3mm] \displaystyle
    \frac{\sqrt{h}} {1024\psi_{a}}  & \mbox{if
      $v \in \left(\frac{1-s}{2}, 1 \right]$.}
  \end{cases}
\end{align}

This function acts as a lower bound on the truncated conductance
profile. 

Consider a pair $(x, y) \in \Omega$ such that $\left\|x-y\right\| \leq \sqrt{\frac{h}{2}}$. Invoking the
decomposition~\eqref{eq:transition_lazy_decomposition} and applying
triangle inequality for $\zeta$-lazy MAO, we have: 
\begin{multline*}
  \left\|\mathcal{T}_{x}-\mathcal{T}_{y}\right\|_{TV} \leq \zeta + (1 - \zeta)\left\|\mathcal{T}^{\textrm{Before-Lazy}}_{x}-\mathcal{T}^{\textrm{Before-Lazy}}_{y}\right\|_{TV}\\
  \leq \zeta +(1-\zeta)\left(\left\|\mathcal{T}^{\textrm{Before-Lazy}}_{x}-\mathcal{P}_{y}\right\|_{TV} + \left\|\mathcal{P}_{x}-\mathcal{P}_{y}\right\|_{TV} + \left\|\mathcal{T}^{\textrm{Before-Lazy}}_{y}-\mathcal{P}_{x}\right\|_{TV} \right) \\
  \leq \zeta +(1-\zeta)\left( \frac{1} {4}+\frac{1}{2} + \frac{1}{4}\right)= 1-\frac{1-\zeta}{4}
\end{multline*}
where step~(i) follows from the bounds~\eqref{eq:proposal_closeness}
and~\eqref{eq:transition_proposal_overlap} from
Lemma~\ref{lem:transition_closeness}.  For $\zeta=\frac{1}{2}$ substituting $\omega=\frac{1}{8}$, $\Delta=\frac{\sqrt{h}}{2}$ and the convex set $\Omega$ into Lemma 2~\cite{chenhmc},
we obtain that:
\begin{align*}
  \Phi_{\Omega}(v) \geq \frac{1}{32}\cdot \min\left\{1,
    \frac{\sqrt{h}}{32\psi_{a}}
    \log^{a}\left(1 + \frac{1}{v}\right)\right\},\quad\text{for }
  v\in \left[0, \frac{1-s}{2} \right].
\end{align*}
This yields that $\Psi_{\Omega}$ acts as a lower bound on the truncated conductance profile~\cite{chenhmc}, we
have that $\tilde{\Phi}_{\Omega}(v) \geq
\frac{\sqrt{h}}{1024\psi_{a}}$ for
$v\in\left[\frac{1-s}{2}, 1\right]$.  Note that the assumption~\ref{itm:assumptionB}
ensures the existence of $\Omega$ such that $\Pi(\Omega) \geq
1- s$ for $s = \frac{\epsilon^2}{3\beta^2}$. We conclude by applying
Lemma 4~\cite{chenhmc} with the convex set $\Omega$ concludes the proof of the theorem.
\subsection{Proof of Theorem \ref{thm_a}}
\label{subsec:Proof}
This is a special case of Theorem \ref{thm_b} with $\gamma=2$, $\Omega=\mathcal{R}_{s,\frac{1}{2}}$, and $\tau(s)=r(s)$.



\subsection{Proof of \eqref{eq:proposal_closeness} in Lemma~\ref{lem:transition_closeness}}
\label{sub:properties_of_the_hmc_proposal}
In this subsection we provide proof for claim \eqref{eq:proposal_closeness}. First we control the total variation distance of the proposal between two points $(x,y)$. To that end, we  apply Pinsker’s inequality~\cite{Cover},  which states that $\left\|\mathcal{P}_{x}-\mathcal{P}_{y}\right\|_{TV} \leq \sqrt{2\textrm{KL}\left(\mathcal{P}_{x}||\mathcal{P}_{y}\right)}$. Given multivariate normal distributions $\mathcal{G}_{1}=\mathcal{N}(\mu_{1},\Sigma)$ and $\mathcal{G}_{2}=\mathcal{N}(\mu_{2},\Sigma)$ he Kullback-Leibler divergence between the two distributions is given by
\begin{align}
    \textrm{KL}\left(\mathcal{G}_{1}||\mathcal{G}_{2}\right)=\frac{1}{2}(\mu_{1}-\mu_{1})^{\intercal}\Sigma^{-1}(\mu_{1}-\mu_{1})
\end{align}

Setting $\mathcal{G}_{1}=\mathcal{P}_{x}$ and $\mathcal{G}_{2}=\mathcal{P}_{y}$ and applying Pinsker’s inequality, we have that 
\begin{align*}
    \left\|\mathcal{P}_{x}-\mathcal{P}_{y}\right\|_{TV} \leq \sqrt{2\textrm{KL}\left(\mathcal{P}_{x}||\mathcal{P}_{y}\right)}=\frac{\left\|\mu_{x}-\mu_{y}\right\|}{\sqrt{2h}}
     & \leq \frac{(1-h)\left\|x-y\right\|}{\sqrt{2h}} \leq \frac{\left\|x-y\right\|}{\sqrt{2h}}
\end{align*}
which concludes the proof of the result of equation \eqref{eq:proposal_closeness}.

\subsection{Proof of ~\texorpdfstring{\eqref{eq:transition_proposal_overlap}}{proposal-proposal-overlap} in Lemma~\ref{lem:transition_closeness}}
\label{sub:proof_of_claim_eq:proposal_closeness}
We now bound the distance between the one-step proposal distribution
$\mathcal{P}_x$ at point $x$ and the one-step transition distribution
$\mathcal{T}^{\textrm{Before-Lazy}}_{x}$ at $x$ obtained after the Metropolis-Hastings correction step (and no lazy step). Following the outline of the proof of  the result of Lemma 3 in~\cite{dwivedi2018log} we have that the total variation distance between the proposal and transition
distribution is given by :
\begin{align}
\label{eq:trans}
    \left\|\mathcal{T}^{\textrm{Before-Lazy}}_{x}-\mathcal{P}_{y}\right\|_{TV}=1-\displaystyle\int_{\mathcal{X}} \min\left\{1,\frac{\pi(z)p_z(x)}{\pi(x)p_x(z)}\right\}dz
    & =1-\mathbb{E}_{z\sim\mathcal{P}_x}\left[ \min\left\{1,\frac{\pi(z)p_z(x)}{\pi(x)p_x(z)}\right\} \right]
\end{align}
where $p_x$ is used to denote the density of the MAO proposal distribution $\mathcal{P}_x=\mathcal{N}(x-h(x-\tilde{x}),2h\mathbb{I}_d)$, and where we have used the following result (showing this is a straightforward exercise, so the proof is omitted).
\paragraph{Total variation distance between distributions when one of them has an atom.}
Let $\mathcal{P}_{1}$ be a distribution admitting a density $p_1$ on $\mathcal{X}$, and let $\mathcal{P}_{2}$ be a distribution which has  an  atom  at $x$ and  admitting  a  density $p_2$ on $\mathcal{X}$. The  total  variation  distance between the distributions $\mathcal{P}_{1}$ and $\mathcal{P}_{2}$ is given by :
\begin{align}
       \left\|\mathcal{P}_{1}-\mathcal{P}_{2}\right\|_{TV}=\frac{1}{2}\left(\mathcal{P}_{2}\{x\} + \displaystyle\int_{\mathcal{X}}|p_1(z)-p_2(z)|dz\right)
\end{align}

Now we continue with the proof of \eqref{eq:transition_proposal_overlap}.
An application of Markov's inequality yields that
\begin{align}
\label{eq:markov}
    \mathbb{E}_{z\sim\mathcal{P}_x}\left[ \min\left\{1,\frac{\pi(z)p_z(x)}{\pi(x)p_x(z)}\right\} \right]\geq \rho\cdot  \mathbb{P}_{z\sim\mathcal{P}_x} \left[\frac{\pi(z)p_z(x)}{\pi(x)p_x(z)}\geq\rho \right] 
\end{align}
for any $\rho\in (0, 1]$.  Thus, to bound the distance $\left\|\mathcal{T}^{\textrm{Before-Lazy}}_{x}-\mathcal{P}_{y}\right\|_{TV}$ it suffices to derive a high probability lower bound on the ratio:
$$a(x,z) =\frac{\pi(z)p_z(x)}{\pi(x)p_x(z)} \quad \textrm{when} \quad z\sim\mathcal{P}_x $$
Thus our goal is  derive a lower bound on the average acceptance probability. To that end we have that 
\begin{equation}
    A(x,y)=\min\left\{1,\frac{\pi(y)Q(y,x)}{\pi(x)Q(x,y)}\right\}=\min\left\{1,\frac{\pi(y)\phi(y)Q(y,x)\phi(x)}{\pi(x)\phi(x)Q(x,y)\phi(y)}\right\}
=\min\left\{1,\frac{\pi(y)\phi(x)}{\pi(x)\phi(y)}\right\}
\end{equation}
$\phi(.)$ is the standard Gaussian density, and keeping in line with the notations of 
subsection~\ref{sub:related_sampling_algorithms} we have that $Q(x,y)=p_y(x)$, and noticing that the proposal $Q$ is reversible w.r.t the standard Gaussian density, the expression simplifies, hence the average acceptance probability for MAO resembles that of IMH  with a standard Gaussian proposal but crucially in our case, when computing the average acceptance probability $z$ is not drawn from the standard Gaussian distribution $\phi$ but from the proposal distribution $\mathcal{P}_x$:
\begin{align}
    \mathbb{E}_{z\sim Q_x}[A(x,z)]=\mathbb{E}_{z\sim Q_x}[\exp(-[f(z)-f(x)+\log(\phi(x))-\log(\phi(z)]_+])=\mathbb{E}_{z\sim Q_x}[\exp{-J_{+}}]
\end{align}
Where we have defined $J$: 
\begin{align}
J := f(z)-f(x)+\log(\phi(x))-\log(\phi(z)) \\
     =f(z)-f(x)-\frac{\left\|x\right\|^2}{2}+\frac{\left\|z\right\|^2}{2}
\end{align}
\\
We wish to lower bound the acceptance probability i.e to upper bound $J$.
To that end we make use of Standard Taylor formula stating that:
\begin{align*}
f(z)&=f(x)+\langle z-x,\nabla f(x)\rangle + \displaystyle\int_0^1 (1-s)\langle x-z, \nabla^2 f(x+s(z-x)) (x-z) \rangle ds\\
&\le f(x)+\langle z-x,\nabla f(x)\rangle + \displaystyle\int_0^1 (1-s) \|\nabla^2 f(x+s(z-x))\| \|x-z\|^2 ds
\end{align*}
Using the condition \eqref{eq:nabla_f}, we have \[\|\nabla^2 f(x+s(z-x))\|\le K_2(1+\|x+s(z-x)\|^{\alpha-2})\le K_2(1+2^{\alpha-3} \|x\|^{\alpha-2}+2^{\alpha-3} s^{\alpha-3}\|z-x\|^{\alpha-2}),\]
and hence
\begin{align}
\nonumber 
f(z)&\le f(x)
+\langle z-x,\nabla f(x)\rangle + \displaystyle\int_0^1 (1-s) K_2(1+2^{\alpha-3} \|x-x^*\|^{\alpha-2}+2^{\alpha-3} s^{\alpha-2}\|z-x\|^{\alpha-2}) \|z-x\|^2 ds\\
&\le f(x)+\langle z-x,\nabla f(x)\rangle+
K_2 \cdot 2^{\alpha-3}\cdot (\|z-x\|^2+\|z-x\|^{\alpha}+\|z-x\|^2\cdot \|x-x^*\|^{\alpha-2})
\end{align}
Taking note of the fact that $z-x=-h(x-x^*)+\sqrt{2h}\xi$ where $\xi\sim\mathcal{N}(0,\mathbb{I}_d)$ we have:
\begin{align*}
    f(z)&\le f(x)-h\langle x-x^*,\nabla f(x)\rangle +\sqrt{2h}\langle \xi,\nabla f(x)\rangle \\
    &+ K_2 \cdot 2^{\alpha-3}\cdot (\|z-x\|^2+\|z-x\|^{\alpha}+\|z-x\|^2\cdot \|x-x^*\|^{\alpha-2})\\
    &\le f(x)+\sqrt{2h}\langle \xi,\nabla f(x)\rangle+K_2 \cdot 2^{\alpha-3}\cdot (\|z-x\|^2+\|z-x\|^{\alpha}+\|z-x\|^2\cdot \|x-x^*\|^{\alpha-2})
\end{align*}
We now make use of the concentration bounds for the $\chi^2$-distribution allowing us to state that the following results hold with high probability that 
$$ (i)  \left\|\xi\right\|\leq M\sqrt{d}  \  \textrm{for} \  M>0 $$
$$ (ii)  \langle v,\xi \rangle\leq M \ \textrm{for} \ \left\|v\right\|=1$$
See the subsection \ref{sub:tails} for more detailed treatment of the tail bounds.
Using these bounds, we can see that with high probability, we have
\begin{align}&\sqrt{2h}\langle \xi,\nabla f(x)\rangle\le \sqrt{2h}M \|\nabla f(x)\|\le \sqrt{2h}M K_1(1+\|x-x^*\|^{\alpha-1}), \text{ and}\\
&\|z-x\|=\|-h(x-x^*)+\sqrt{2h}\xi\|\le h\|x-x^*\|+\sqrt{2h} M\sqrt{d}.
\end{align}
Hence with high probability, we have
\begin{align*}
    f(z)&\le f(x)-h\langle x-x^*,\nabla f(x)\rangle +\sqrt{2h}\langle \xi,\nabla f(x)\rangle \\
    &+ K_2 \cdot 2^{\alpha-3}\cdot (\|z-x\|^2+\|z-x\|^{\alpha}+\|z-x\|^2\cdot \|x-x^*\|^{\alpha-2})\\
    &\le f(x)+\sqrt{2h}M K_1(1+\|x-x^*\|^{\alpha-1})+K_2 \cdot 2^{\alpha-3}\cdot (\|z-x\|^2+\|z-x\|^{\alpha}+\|z-x\|^2\cdot \|x-x^*\|^{\alpha-2})
\end{align*}
In the case when $f$ satisfies Assumption \ref{itm:assumptionB}, we have that $x-x^*\le C d^{\frac{1}{\gamma}}$ for every $x\in \Omega$ and this set has probability $\Pi(\Omega)\ge 1-s$. In this case, let 
\[\omega:=\max
\left(\frac{2 (\alpha-1)}{\gamma},\frac{\gamma+\alpha-2}{\gamma}\right).\]
Then with the choice $h \propto d^{-\omega}$, it is easy to check that $\sqrt{2h}M K_1(1+\|x-x^*\|^{\alpha-1})=O(1)$, $\|z-x\|=O(1)$ and $\|z-x\|^2\cdot \|x-x^*\|^{\alpha-2}=O(1)$, hence $f(z)-f(x)=O(1)$.

There are two additional terms due to the Gaussian density that we need to control :
$$ \left\|z\right\|^2-\left\|x\right\|^2=\left\|x-hx+\sqrt{2h}\xi\right\|^2-\left\|x\right\|^2=-2\langle x,hx-\sqrt{2h}\xi\rangle+ \left\|hx-\sqrt{2h}\xi\right\|^2$$
$$\leq2h^2\left\|x\right\|^2+4h\left\|\xi\right\|^2-2h\left\|x\right\|^2+4\sqrt{h}|\langle x,\xi\rangle|$$
$$\leq 4\sqrt{h}|\langle x,\xi\rangle|+4h\left\|\xi\right\|^2\lesssim\sqrt{h}d^{\frac{1}{\gamma}}+hd.$$
Recall that we assume that $\gamma\ge 2$, hence this term can be controlled as long as $h=O(d^{-1})$, which is satisfied for our choice $h\propto d^{-\omega}$.

It is not hard to see that using our choice of step size $h$ as detailed in in the assumptions of Theorem~\ref{thm_b}, and the $\chi^2$ tail bounds we have that: 
\begin{align}
    \mathbb{P}\left[-J\geq -\frac{1}{16}\right] \geq 1-\frac{1}{16}
\end{align}
thus by plugging this bound and using the Markov inequality we have:
\begin{align}
     \mathbb{P}\left[A(x,z)\geq \exp\left(-\frac{1}{16}\right)\right] \geq 1-\frac{1}{16}
\end{align}
Thus, we have derived a desirable high probability lower bound on the accept-reject ratio. Substituting $\alpha=\exp\left(-\frac{1}{16}\right)$ in inequality \eqref{eq:markov} and using the fact that $e^{-\frac{1}{16}}\geq 1-\frac{1}{16}$ we find that:
\begin{align*}
     \mathbb{E}_{z\sim\mathcal{P}_x}\left[ \min\left\{1,\frac{\pi(z)p_z(x)}{\pi(x)p_x(z)}\right\} \right]\geq 1-\frac{1}{8}
\end{align*}
which when plugged in equation \eqref{eq:trans} implies that
\begin{align*}
      \left\|\mathcal{T}^{\textrm{Before-Lazy}}_{x}-\mathcal{P}_{y}\right\|_{TV}\leq \frac{1}{8} \qquad \textrm{for any} \quad x \in \Omega,
\end{align*}
which concludes the proof of claim~\texorpdfstring{\eqref{eq:transition_proposal_overlap}}{proposal-proposal-overlap} in Lemma~\ref{lem:transition_closeness}.
In the next Section we derive results for the $\chi^2$ distribution concentration bounds that allowed us to state high probability upper bounds used in our proof outline. 
\subsection{$\chi^2$ Tail Bounds}
\label{sub:tails}
 In this section we will prove the concentration bounds for the $\chi^2$ distribution, we first state the standard $\chi^2$-bound, used in~\cite{dwivedi2018log} proof given by : 
$$ \mathbb{P} \left(\left\|\xi\right\|^2\leq d \tau_{\epsilon} \right)\geq1-\frac{\epsilon}{32}$$
where : 
$$ \tau_{\epsilon}= 1+2\sqrt{\log(32/\epsilon)} +2\log(\epsilon/32)$$
which in turn yields : 
$$\mathbb{P} \left(\left\|\xi\right\|^r\leq d^\frac{r}{2} \tau_{\epsilon}^\frac{r}{2} \right)\geq1-\frac{\epsilon}{32}$$
\newline
This allows us to state tail bounds concentration inequalities with high probability for $\left\|\xi\right\|^{r}$, we now derive a high probability bound in order to control the dot product $\langle\nabla f(x), \xi\rangle$ on the $\Omega$,  i.e for $\left\|x-x^*\right\|\leq \tau(s)d^{\frac{1}{\gamma}}$  we have for $\xi\sim\mathcal{N}(0,\mathbb{I}_d)$ : 
\begin{multline*}\mathbb{P}\left(|\langle\nabla f(x), \xi\rangle|\leq t\right)=\mathbb{P}\left(\left\|\nabla f(x)\right\||\langle\frac{\nabla f(x)}{\left\|\nabla f(x)\right\|},\xi\rangle|\leq t\right)\geq\mathbb{P}\left(\tau(s)d^{\frac{\alpha}{\gamma}}|\langle\frac{\nabla f(x)}{\left\|\nabla f(x)\right\|},\xi\rangle|\leq t\right)
\\ \geq 1-2\exp\left(-\frac{t^2}{\tau(s)d^{\frac{2\alpha}{\gamma}}}\right)\geq 1-\epsilon 
\end{multline*}
the desired result can be achieved by setting : 
$$ t= d^{\frac{\alpha}{\gamma}}\sqrt{2\tau(s)\log(2/\epsilon)}$$
where we used the fact that for $\zeta\sim\mathcal{N}(0,1)$ we have: 
$$ \forall x\geq0 \ \mathbb{P}\left(|\zeta|\geq x\right)\leq 2e^{-\frac{x^2}{2}}.$$

\subsection{Proof of Theorem \ref{thm:opti_mao}}
In this section we provide proof of Theorem \ref{thm:opti_mao}. In this regime, our offline Optimization fails too learn the exact mode of the distribution, instead it outputs a estimate $\tilde{x}$ which is $\delta$ off of the true mode, i.e $\left\|\tilde{x}\right\|\leq \delta$. 
In such regime, MAO will have as a proposal that is similar to the one used in MALA with the slight difference of the gradient step being calculated w.r.t a Gaussian distribution centered at the learned mode, that is $\tilde{x}$, this in turn modifies the accept reject step 
 \begin{equation*}
      A(x,z)=\min\left\{1,\frac{\pi(z)\tilde{Q}(z,x)}{\pi(x)\tilde{Q}(x,z)}\right\}=\min\left\{1,\frac{\pi(z)\tilde{\phi}(z)\tilde{Q}(z,x)\tilde{\phi}(x)}{\pi(x)\tilde{\phi}(x)\tilde{Q}(x,z)\tilde{\phi}(z)}\right\}
=\min\left\{1,\frac{\pi(y)\tilde{\phi}(x)}{\pi(x)\tilde{\phi}(y)}\right\}
\end{equation*}
Where we have $z\sim \tilde{Q}(x,z)=\mathcal{N}(x-h(x-\tilde{x}),2h\mathbb{I}_d)$ and $\tilde{\phi}=\mathcal{N}(\tilde{x},1)$, and where we have used the fact that $\tilde{Q}(x,z)$ is reversible w.r.t $\tilde{\phi}$ thus simplifying our expression. In similar fashion to the outline of  our proof of Theorem \ref{thm_b} in subsection \ref{sub:proof_thm2}, the key idea is to  control the average acceptance probability:
\begin{equation}
    \mathbb{E}_{\tilde{z}\sim \tilde{Q}_x}[A(x,\tilde{z})]=\mathbb{E}_{\tilde{z}\sim \tilde{Q}_x}[\exp-[f(\tilde{z})-f(x)+\log(\tilde\phi(x))-\log(\tilde{\phi}(\tilde{z})]_+]
\end{equation}
Where in this instance we have:  $\tilde{z}=x-h(x-\tilde{x})+\sqrt{2h}\xi$, keeping in line with the notations of subsection \ref{sub:proof_of_claim_eq:proposal_closeness} we define :
\begin{align*}
     \tilde{J}:= f(\tilde{z})-f(x)+\log(\tilde{\phi}(x))-\log(\tilde{\phi}(\tilde{z})) \\
=f(\tilde{z})-f(x)-\frac{\left\|x-\tilde{x}\right\|^2}{2}+\frac{\left\|\tilde{z}-\tilde{x}\right\|^2}{2} \\
=f(\tilde{z})-f(x)-\frac{\left\|x\right\|^2}{2}+\frac{\left\|\tilde{z}\right\|^2}{2}+\tilde{x}^{\intercal}(x-\tilde{z})
\end{align*}
Our goal is to upper bound $\tilde{J}$. To that  again we will control each term individually:
 
 \begin{subequations}
    \begin{align}
        \label{eq:q_j_recursion_bound}
          f(\tilde{z})-f(x)\\
        \label{eq:q_j_recursion_bound1}
         \frac{\left\|\tilde{z}\right\|^2}{2}-\frac{\left\|x\right\|^2}{2}\\
        \label{eq:q_j_diff_recursion_bound3}
         \tilde{x}^{\intercal}(x-\tilde{z})
    \end{align}
\end{subequations}
In order to control  term \eqref{eq:q_j_diff_recursion_bound3} we  use Cauchy-Schwartz inequality:
\begin{align*}
    | \tilde{x}^{\intercal}(x-\tilde{z})|\leq \left\|\tilde{x}\right\|\left\|-h(x-\tilde{x})+\sqrt{2h}\xi\right\|\leq\delta\left(h(\delta +\tau(s)d^{\frac{1}{\gamma}})+M\sqrt{2dh}\right):=A_1(h,\delta)
\end{align*}
where in we have used the fact that $\left\|\tilde{x}\right\|\leq \delta$ and $\left\|x\right\|\leq \tau(s)d^{\frac{1}{\gamma}}$ and $\left\|\xi\right\|\leq M\sqrt{d}$ where the quantities have been defined in the previous sections.
\\
As for term \eqref{eq:q_j_recursion_bound1} a second  use of Cauchy-Schwartz yields 
\begin{align*}
    \frac{\left\|\tilde{z}\right\|^2}{2}-\frac{\left\|x\right\|^2}{2} = \frac{\left\|z+h\tilde{x}\right\|^2}{2}-\frac{\left\|x\right\|^2}{2}=\frac{\left\|z\right\|^2}{2}-\frac{\left\|x\right\|^2}{2}+h\tilde{x}^{\intercal}z+\frac{h^2\left\|\tilde{x}\right\|^2}{2} \\
    \leq \frac{\left\|z\right\|^2}{2}-\frac{\left\|x\right\|^2}{2}+ h\delta(\tau(s)d^{\frac{1}{\gamma}}(1+h)+M\sqrt{2dh})+\frac{h^2\delta^2}{2}
\end{align*}
thus we have proved that
$$\frac{\left\|\tilde{z}\right\|^2}{2}-\frac{\left\|x\right\|^2}{2}\leq \frac{\left\|z\right\|^2}{2}-\frac{\left\|x\right\|^2}{2}+A_{2}(h,\delta)$$
where
$$A_2(h,\delta)=h\delta(\tau(s)d^{\frac{1}{\gamma}}(1+h)+M\sqrt{2dh})+\frac{h^2\delta^2}{2}$$

In order to control term \eqref{eq:q_j_recursion_bound} we use Taylor's formula to obtain
$$f(\tilde{z})-f(x)=f(\tilde{z})-f(z)+f(z)-f(x),$$
before developing the expression as the following: 
\begin{align*}
    f(\tilde{z})-f(z)\le\langle\tilde{z}-z,\nabla f(z)\rangle +\displaystyle\int_0^1 (1-s)\left\|\nabla^2 f(z+s(\tilde{z}-z)\right\|\left\|\tilde{z}-z\right\|^2ds \\
    = \langle h\tilde{x}, \nabla f(z)\rangle+\displaystyle\int_0^1 (1-s)\left\|\nabla^2 f(z+hs\tilde{x})\right\|\left\|h\tilde{x}\right\|^2ds\\
    \leq h\delta K_{1}(1+\left\|z\right\|^{\alpha-1})+h^2\delta^2\displaystyle\int_0^1 (1-s)K_{2}(1+\left\|z+hs\tilde{x}\right\|^{\alpha-2})ds\\
    \leq h\delta K_{1}(1+2^{\alpha-2}\left((1-h)^{\alpha-1}(\tau(s)d^{\frac{1}{\gamma}})^{\alpha-1})+(M\sqrt{2dh})^{\alpha-1}\right))\\ +h^2\delta^2\left(\frac{K_{2}}{2}+2^{\alpha-3}\frac{K_{2}}{2}\left\|z\right\|^{\alpha-2}+2^{\alpha-3} \frac{K_2}{\alpha(\alpha-1)}(h\delta)^{\alpha-2}\right)
    \\ \leq h\delta K_{1}\left(1+2^{\alpha-2}\left((1-h)^{\alpha-1}(\tau(s)d^{\frac{1}{\gamma}})^{\alpha-1})+(M\sqrt{2dh})^{\alpha-1}\right)\right) \\ +h^2\delta^2\left(\frac{K_{2}}{2}+2^{\alpha-3}\frac{K_{2}}{2}2^{\alpha-3}\left(((1-h)\tau(s)d^{\frac{1}{\gamma}})^{\alpha-2}+(M\sqrt{2dh})^{\alpha-2}\right)+2^{\alpha-3} \frac{K_2}{\alpha(\alpha-1)}(h\delta)^{\alpha-2}\right)\\
    \leq h\delta K_{1}\left(1+2^{\alpha-2}\left((1-h)^{\alpha-1}(\tau(s)d^{\frac{1}{\gamma}})^{\alpha-1})+(M\sqrt{2dh})^{\alpha-1}\right)\right)\\
    \frac{K_2 h^2\delta^2}{2}\left(1+2^{2(\alpha-3)}\left(((1-h)\tau(s)d^{\frac{1}{\gamma}})^{\alpha-2}+(M\sqrt{2dh})^{\alpha-2}\right)+ \frac{2^{\alpha-3}}{\alpha(\alpha-1)}(h\delta)^{\alpha-2}\right) :=A_{3}(h,\delta).
\end{align*}
Hence we have proved that  
$$f(\tilde{z})-f(x)\leq f(z)-f(x)+ A_{3}(h,\delta)$$
Piecing it all together we have proved that 
$$\tilde{J}\leq J+ A(h,\delta),$$
where we have :
$$A(h,\delta)=A_{1}(h,\delta)+A_{2}(h,\delta)+A_{3}(h,\delta)$$

Noticing that  
$$\lim_{\delta\to0} A(h,\delta)=0$$

In accordance with the outline of the proof of Theorem \ref{thm_b} and when inspecting the preponderant terms in each individual $A_{i}(h,\delta)$ for $i\in \{1,2,3\}$ it is straightforward that when setting the error tolerance $\delta=\mathcal{O}(\frac{1}{h}\log(\frac{1}{\epsilon}))$ for $\epsilon\in (0,1)$ we have that 
\begin{align}
    \mathbb{P}\left[-\tilde{J}\geq -\frac{1}{16}+ \mathcal{O}(\log(\frac{1}{\epsilon}))\right] \geq 1-\frac{1}{16}
\end{align}

Thus we are able to recover the mixing times obtained in Theorem \ref{thm_a} and Theorem \ref{thm_b}, thus the effect of the error incurred by the off-line optimization step in our MAO algorithm is not prohibitive and we are able to recover mixing times of the same order of those obtained when we have access to the exact mode of our target distribution (i.e in the no-learning regime).

\subsection{Examples of distributions satisfying Assumption \ref{itm:assumptionB}}
We now prove a tighter concentration on a convex set $\mathcal{R}_{s,\alpha}$ for a subclass of target distributions  $\mathcal{E}(\alpha,m)$, the result is given as the following.
\begin{lem}
\label{lem_RS}
Consider the class of target distributions $\Pi_{\alpha}$ with densities $$p_{\alpha}(x)=\frac{e^{-\left\|x\right\|^{\alpha}}}{Z_{\alpha}}$$
there exists a poly-logarithmic function $\tau$ such that the ball $\mathcal{R}_{s,\alpha}=\mathbb{B}\left(0,\tau(s)d^{\frac{1}{\alpha}}\right)$ satisfies that
$$ \Pi_{\alpha}(\mathcal{R}_{s,\alpha})\geq 1-s \quad \textrm{for all} \quad s \in (0,1).$$
\end{lem}
\begin{proof}
We start by showing the results for the target distributions $\Pi_{\alpha}$. In order to prove the result of Lemma ~\ref{lem_RS} we will derive calculations of the probability over a collection of convex sets.
More precisely, let us at a first glance try and estimate the flow of probability  on a ball of radius $\rho$ i.e: 
$$I_{\rho}=\frac{1}{Z_{\alpha}}\displaystyle\int_{\mathcal{B}(0,\rho)} e^{-\left\|x\right\|^{\alpha}}dx $$
Noting that this is a radial integral we can pass into spherical coordinates by making the coordinate change yielding:
$$I_{\rho}=\frac{d\omega_d}{Z_{\alpha}}\displaystyle\int_{0}^{\rho} e^{-r^{\alpha}}r^{d-1}dr $$
Where $\omega_d$ is the volume of a unit ball in $\mathbb{R}^d$ $$\omega_d=\frac{\pi^{\frac{d}{2}}}{\Gamma(\frac{d}{2}+1)}$$
 we then make the variable change $u=r^{\alpha}$ to get: 
 $$I_{\rho}=\frac{d\omega_d}{\alpha Z_{\alpha}}\displaystyle\int_{0}^{\rho^{\alpha}} e^{-u}u^{\frac{d-\alpha}{\alpha}}du$$
 Let us at this stage calculate the value of $Z_{\alpha}$:
 $$Z_{\alpha}=\frac{d\omega_d}{\alpha}\displaystyle\int_{0}^{+\infty}e^{-u}u^{\frac{d-\alpha}{\alpha}}du=\frac{d\omega_d\Gamma(\frac{d}{\alpha})}{\alpha}.$$
 At this stage we recognize the form of a gamma distribution of $X\sim\Gamma(\frac{d}{\alpha},1)$, we have: 
  $$I_{\rho}=\mathbb{P}\left(X\leq\rho^{\alpha}\right)$$
  Now we calculate using  the tail behavior of sub-Gamma random variables (see Section 2.4 of \cite{boucheron2013concentration}) to obtain : 
  $$\rho^{\alpha}=\mathbb{E}(X)+t+\sqrt{\frac{2d}{\alpha}t}=\frac{d}{\alpha}+t+\sqrt{\frac{2d}{\alpha}t}$$
  $$\mathbb{P}\left(X\leq \rho^{\alpha}\right)\geq 1-e^{-t}$$
Now if we write $\rho=\tau(s)d^{\frac{1}{\alpha}}$ and we set $1-e^{-t}=1-s$ this yields in turn:
$$t=-\log\left(s\right)$$ 
and: 
$$\tau(s)^{\alpha}=\frac{1}{d}\left(\frac{d}{\alpha}+\log\left(1/s\right)+\sqrt{\frac{2d}{\alpha}\log(1/s)}\right)=\frac{1}{\alpha}+\frac{1}{d}\log\left(1/s\right)+\sqrt{\frac{2}{d\alpha}\log(1/s)}$$
i.e : 
\begin{align}
\label{eq:tau}
    \tau(s)=\left(
    \frac{1}{\alpha}+\frac{1}{d}\log\left(1/s\right)+\sqrt{\frac{2}{d\alpha}\log(1/s)}
    \right)^{\frac{1}{\alpha}}.
\end{align}
\end{proof}

\subsection{Proof of Proposition \ref{prop:feasiblestart}}\label{sec:proof_of_feasiblestart}
Since both $\mu_0$ and $\Pi$ have a density, we need to check that $\sup_{x\in \mathbb{R}^d}\frac{\mu_0(\mathrm{d}x)}{\Pi(\mathrm{d}x)}\le \beta$.
Without loss of generality, assume that $f(0)=0$. By integrating the condition \eqref{eq:nabla_f} twice, and using the facts that $f$ is minimized at $0$ and $f(0)$, it follows that $f(x)\le f_0(x)$ for every $x$.
Let $\mu_0(x)=\frac{e^{-f_0(x)}}{Z_0}$ and $\Pi(x)=\frac{e^{-f(x)}}{Z}$, then it is easy to see that
$\sup_{x\in \mathbb{R}^d}\frac{\mu_0(x)}{\Pi(x)}\le \frac{Z}{Z_0}$. Hence we only need to bound the normalizing constants. Using the $m$-strong convexity, it follows that $f(x)\ge \frac{m}{2}\|x\|^2$, and so 
\[Z=\int_{x\in \mathbb{R}^d}e^{-f(x)} \mathrm{d}x\le \int_{x\in \mathbb{R}^d}e^{-\frac{m}{2}\|x\|^2} \mathrm{d}x = \left(\frac{2\pi}{m}\right)^{d/2}.\]
For the other normalizing constant $Z_0$, using the fact that $\alpha\ge 2$, we have
\begin{align*}
Z_0&=\int_{x\in \mathbb{R}^d}e^{-f_0(x)} \mathrm{d}x= \int_{x\in \mathbb{R}^d}e^{-K_2 \left(\frac{\|x\|^2}{2}+\frac{\|x\|^{\alpha}}{\alpha(\alpha-1)}\right)} \mathrm{d}x \\
&\ge \int_{x\in \mathbb{R}^d}e^{-K_2 \left(\frac{1}{2}+2\frac{\|x\|^{\alpha}}{\alpha(\alpha-1)}\right)} \mathrm{d}x
\intertext{ by change of variables}
&=e^{-\frac{K_2}{2}} \left(\frac{2 K_2}{\alpha(\alpha-1)}\right)^{-\frac{d}{\alpha}}\int_{y\in \mathbb{R}^d}e^{-\|y\|^{\alpha}} \mathrm{d}y.
\end{align*}
In the proof of Lemma \ref{lem_RS}, we show that
$\int_{x\in \mathbb{R}^d}e^{-\|x\|^{\alpha}} \mathrm{d}x=\frac{d\cdot \pi^{\frac{d}{2}}\cdot \Gamma(\frac{d}{\alpha})}{\alpha \Gamma(\frac{d}{2}+1)}$.
The claim of the proposition now follows by rearrangement.

\section{Discussion}
\label{sec:discussion}
In  this  paper, we proposed a novel Metropolized algorithm (MAO), we  also derived  non-asymptotic  bounds  on  the  mixing  time  of  our  Metropolized  algorithm on a class of thin tailed distributions.
Our algorithm id based on a two-phase scheme: (1) a proposal step followed by (2) an accept-reject
step. Our results show that our algorithm MAO yields better mixing times than that of MALA on the class of thin tailed potentials, our numerical results highlight the fact that MAO mixes significantly faster than MALA, and in some cases MAO is able to reach a stationary state whereas MALA fails to do so. Furthermore, we have leveraged the framework outlined in~\cite{chenhmc} to improve the dependence on the warmness of MAO from $\log\beta$ to $\log\log\beta$. Moreover, we have proved that our mixing time bounds remain unchanged when our off-line Optimization step fails to capture the true mode of the target distribution, thus further highlighting the merits of MAO over MALA. 

Several fundamental questions arise from our work. All of our results
are upper bounds on mixing time, and the question on the sharpness of our mixing times, and the optimal choice of the MAO step size remain an open question for future work. One potential path worth inspecting, is deriving explicit dimension dependant warm distributions for MAO, that may yield improved dependency on the dimension $d$.

Another open question is that of relaxing the assumptions on the class of target distribution, namely that of strong convexity and by extension that of convexity, thus relying only on imposing  growth assumption on the target distribution. It is an interesting question to leverage the large body of literature on Optimization schemes namely the recent work on relative smoothness and relative strong convexity~\cite{Maddison} to extend our results to a larger class of distributions such as heavy tailed distributions. We believe that MAO has a variety of applications and can be refined to outperform state of the art sampling algorithms, which we leave for future research.
\bibliographystyle{unsrtnat}
\bibliography{references} 
\end{document}